\documentclass[a4paper,fleqn,american]{elsarticle}
\usepackage{babel}

\usepackage{todonotes}
\usetikzlibrary{shapes,shapes.geometric,backgrounds,calc,fit,positioning}

\usepackage[utf8]{inputenc}
\usepackage{graphicx}

\usepackage[ruled,vlined]{algorithm2e}

\usepackage{amsmath,amssymb}
\usepackage{mathtools}
\usepackage{commath}
\usepackage{faktor} 
\usepackage[kerning=true]{microtype} 
\usepackage{ntheorem}
\usepackage{nicefrac}
\usepackage{enumerate}
\usepackage{paralist}
\usepackage{colonequals}

\usepackage[hidelinks]{hyperref}
\usepackage[all]{hypcap}
\usepackage{natbib}
\usepackage[noabbrev]{cleveref}

\usepackage{booktabs}
\usepackage{arydshln}

\usepackage{fca, newdrawline}

\crefalias{problem}{prob}

\newcommand{\Scon}{\mathbb{S}}
\newcommand{\Tcon}{\mathbb{T}}
\newcommand{\Ocon}{\mathbb{O}}
\newcommand{\Sh}{\mathfrak{S}}
\newcommand{\SH}{\underline{\mathfrak{S}}}

\DeclareMathOperator{\id}{id}

\newcommand*{\logeq}{\ratio\Leftrightarrow}

\DeclareMathOperator{\Ext}{Ext}
\DeclareMathOperator{\Int}{Int}

\DeclareMathOperator{\app}{\mid\,}
\DeclareMathOperator{\Var}{Var}

\newtheorem{theorem}{Theorem}
\newtheorem{lemma}[theorem]{Lemma}
\newtheorem{definition}[theorem]{Definition}
\newproof{proof}{Proof}
\newtheorem{proposition}[theorem]{Proposition}
\newtheorem{problem}[theorem]{Problem}
\newtheorem{corollary}[theorem]{Corollary}

\let\cref\Cref

\begin{document}
\let\WriteBookmarks\relax
\def\floatpagepagefraction{1}
\def\textpagefraction{.001}
\begin{frontmatter}

\title
{On the Lattice of Conceptual Measurements}

\author[1,2]%
{Tom Hanika}
\ead{tom.hanika@cs.uni-kassel.de}

\author[1,2]{Johannes Hirth}
\ead{hirth@cs.uni-kassel.de}


\address[1]{Knowledge \& Data Engineering Group, University of Kassel, Germany}
\address[2]{Interdisciplinary Research Center for Information System Design,
  U Kassel, Germany}

\nonumnote{Authors are given in alphabetical order.
  No priority in authorship is implied.}

\begin{abstract}
  We present a novel approach for data set scaling based on
  scale-measures from formal concept analysis, i.e., continuous maps
  between closure systems, and derive a canonical
  representation. Moreover, we prove said scale-measures are lattice
  ordered with respect to the closure systems. This enables exploring
  the set of scale-measures through by the use of meet and join
  operations. Furthermore we show that the lattice of scale-measures
  is isomorphic to the lattice of sub-closure systems that arises from
  the original data. Finally, we provide another representation of
  scale-measures using propositional logic in terms of data set
  features. Our theoretical findings are discussed by means of
  examples.
\end{abstract}


\end{frontmatter}

\section{Introduction}
The discovery and analysis of patterns and dependencies in the realm
of data science does strongly depend on the measurement of the data.
Each data set is subject to one or more scales of
measure~\cite{stevens1946theory}, i.e., maps from the data into
variable of some (mathematical) space, e.g., the real line, an ordered
set, etc. Beyond that, almost every data set is further scaled prior
to (data)processing to meet the requirements of the employed data
analysis method, such as the introduction of artificial metrics, the
numerical representation of nominal features, etc. This scaling is
usually accompanied by a grade of detail, which in turn is becoming
more and more of a problem for data science tasks as the availability
of features increases and their human explainability decreases. Often
used methods to deal with this problem from the field of machine
learning, such as \emph{principal component analysis}, do enforce
particular, possible inapt, levels of measurement, e.g., food tastes
represented by real numbers, and amplify the problem for
explainability.

Therefore, understanding the set of possible scaling maps, identifying
its (algebraic) properties, and deriving to some extent human
explainable control over it, is a pressing problem. This is especially
important since found patterns and dependencies may be artifacts of
some scaling map and may therefore corrupt any subsequent task,e.g.,
classification tasks.
In the case Boolean data sets the field of \emph{formal concept
  analysis} provides a well-formalized, yet insufficiently studied,
approach for mathematically grasping the process of data scaling,
called \emph{scale-measure} maps. These maps are continuous with
respect to the closures systems that emerge from the original Boolean
data set and scale, which resembles also a Boolean data set, i.e., the
preimage of a closed set is closed. Equipped with this notion for data
scaling we discover and characterize consistent scale-refinements and
derive a theory that is able to provide new insights to data sets by
comparing different scale-measures. Building up on this we prove that 
the set of all scale-measures bears a lattice structure and we show
how to transform scale-measures using lattice operations. Moreover, we
introduce an equivalent representation of scale-measures using
propositional logic expressions and how they emerge naturally while
scaling data. 

Altogether, we present methods that are able to generate different
conceptual measurements of a data set by computing meaningful features
such that they are consistent with the conceptual knowledge of the
original data set.

\section{Scales and Measurement}\label{smeasures}
\subsection{Measurements and Categorical Data}
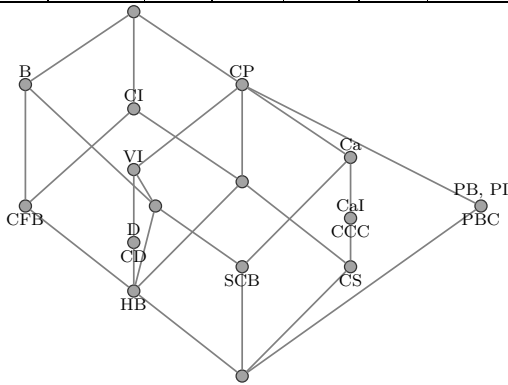
\begin{figure}[t]
  \label{bjice}
  \centering
    \scalebox{0.55}{
      \hspace{-1.4cm}
      \begin{cxt}
        \cxtName{}
        \att{\shortstack{Brownie\\ (B)}}
        \att{\shortstack{Peanut\\ Butter (PB)}}
        \att{\shortstack{Peanut\\ Ice (PI)}}
        \att{\shortstack{Caramel\\ (Ca)}}
        \att{\shortstack{Caramel\\ Ice (CaI)}}
        \att{\shortstack{Choco\\ Ice (CI)}}
        \att{\shortstack{Choco\\ Pieces (CP)}}
        \att{\shortstack{Dough\\ (D)}}
        \att{\shortstack{Vanilla\\ (V)}}
        \obj{x....x...}{\shortstack{Fudge Brownie (FB)\\ \ }}
        \obj{......xxx}{\shortstack{Cookie Dough (CD)\\ \ }}
        \obj{x....xxxx}{\shortstack{Half Baked (HB)\\ \ }}
        \obj{...xxxx..}{\shortstack{Caramel Sutra (CS)\\ \ }}
        \obj{...xx.x..}{\shortstack{Caramel Chew\\ Chew (CCC)}}
        \obj{.xx...x..}{\shortstack{Peanut Butter\\ Cup (PBC)}}
        \obj{x..x..x.x}{\shortstack{Salted Caramel\\ Brownie (SCB)}}
      \end{cxt}}  

    \scalebox{0.8}{\colorlet{mivertexcolor}{black!80}
\colorlet{jivertexcolor}{black!80}
\colorlet{vertexcolor}{black!80}
\colorlet{bordercolor}{black!80}
\colorlet{linecolor}{gray}
\tikzset{vertexbase/.style={semithick, shape=circle, inner sep=2pt, outer sep=0pt, draw=bordercolor},%
  vertex/.style={vertexbase, fill=vertexcolor!45},%
  mivertex/.style={vertexbase, fill=mivertexcolor!45},%
  jivertex/.style={vertexbase, fill=jivertexcolor!45},%
  divertex/.style={vertexbase, top color=mivertexcolor!45, bottom color=jivertexcolor!45},%
  conn/.style={-, thick, color=linecolor}%
}
\tikzstyle{n} = [text width=2.5cm,align=center]
\begin{tikzpicture}[scale=0.35,font=\footnotesize]
  \begin{scope} 
    \begin{scope} 
      \foreach \nodename/\nodetype/\xpos/\ypos in {%
        0/vertex/1.6624628793523328/5.158005322212052,
        1/jivertex/-3.428124546751512/9.18431467431844,
        2/jivertex/1.6624628793523328/10.334688774920266,
        3/mivertex/6.7530503054561795/10.334688774920266,
        4/vertex/-3.428124546751512/11.485062875522091,
        5/mivertex/6.7530503054561795/12.635436976123916,
        6/jivertex/-8.518711972855357/13.21062402642483,
        7/vertex/-2.4100070615307425/13.21062402642483,
        8/vertex/12.879872702001563/13.21062402642483,
        9/vertex/1.6624628793523328/14.360998127026654,
        10/vertex/-3.428124546751512/14.936185177327568,
        11/vertex/6.7530503054561795/15.51137222762848,
        12/vertex/-3.428124546751512/17.81212042883213,
        13/vertex/-8.518711972855357/18.962494529433954,
        14/vertex/1.6624628793523328/18.962494529433954,
        15/vertex/-3.428124546751512/22.41361683123943
      } \node[\nodetype] (\nodename) at (\xpos, \ypos) {};
    \end{scope}
    \begin{scope} 
      \path (3) edge[conn] (5);
      \path (13) edge[conn] (15);
      \path (9) edge[conn] (14);
      \path (6) edge[conn] (13);
      \path (10) edge[conn] (14);
      \path (0) edge[conn] (1);
      \path (14) edge[conn] (15);
      \path (5) edge[conn] (11);
      \path (4) edge[conn] (10);
      \path (11) edge[conn] (14);
      \path (9) edge[conn] (12);
      \path (7) edge[conn] (13);
      \path (1) edge[conn] (6);
      \path (8) edge[conn] (14);
      \path (0) edge[conn] (2);
      \path (3) edge[conn] (9);
      \path (2) edge[conn] (11);
      \path (1) edge[conn] (9);
      \path (0) edge[conn] (8);
      \path (6) edge[conn] (12);
      \path (2) edge[conn] (7);
      \path (12) edge[conn] (15);
      \path (7) edge[conn] (10);
      \path (1) edge[conn] (4);
      \path (0) edge[conn] (3);
      \path (1) edge[conn] (7);
    \end{scope}
    \begin{scope} 
      \foreach \nodename/\labelpos/\labelopts/\labelcontent in {%
        1/below/n/{HB},
        2/below/n/{SCB},
        3/below/n/{CS},
        4/below/n/{CD},
        4/above/n/{D},
        5/below/n/{CCC},
        5/above/n/{CaI},
        6/below/n/{CFB},
        8/below/n/{PBC},
        8/above/n/{PB, PI},
        10/above/n/{VI},
        11/above/n/{Ca},
        12/above/n/{CI},
        13/above/n/{B},
        14/above/n/{CP}
      } \coordinate[label={[\labelopts]\labelpos:{\labelcontent}}](c) at (\nodename);
    \end{scope}
  \end{scope}
\end{tikzpicture}}



  \caption{This Figure shows a Ben and Jerry's context and its concept
    lattice.}
  \label{fig:bj1}
\end{figure}

Formalizing and understanding the process of \emph{measurement} is, in
particular in data science, an ongoing
discussion. \emph{Representational Theory of Measurement}
(RTM)~\cite{suppes1989foundations,luce1990foundations} reflects the
most recent and widely acknowledged current standpoint on this. RTM
relies on homomorphisms from an (empirical) relational structure
$\mathbf{E}=(E,(R_i)_{i\in I})$ to a numerical relational structure
$\mathbf{B}=(B,(S_i)_{i\in I})$, very well explained by
J.~Pfanzagl~\cite{pfanzagl1971theory}, where $B$ is often chosen to be
the real line $\mathbb{R}$ or a $n$ dimensional vector space on
it. However, it might be beneficial to allow for other, more algebraic
(measurement) structures~\cite[p. 253]{roberts1984measurement}. This
is particularly true in cases where the empirical data does not allow
for a meaningful measurement into the \emph{ratio} level
(cf~\cite{stevens1946theory}), e.g., taxonomic ranks in biology or
types of faults in software engineering. Both examples are instances
of \emph{categorical data}, which is classified to the \emph{nominal
  level} with respect to S.~S.~Stevens~\cite{stevens1946theory}. If
such data is also naturally equipped with an rank order relation,
e.g., the Likert scale or school grades, it is situated on the
\emph{ordinal level}.

A mathematical framework well equipped for the nominal as well as
the ordinal level is formal concept analysis (FCA)~\cite{Wille1982,
  fca-book}. In FCA we represent data in the form of \emph{formal
  contexts} as see~\cref{fig:bj1} (top). A formal context is a triple
$(G,M,I)$ with $G$ beeing a finite set of object, $M$ beeing a finite
set of attributes and $I \subseteq G \times M$ an incidence relation
between them. With $(g,m) \in I$ means that object $g$ has attribute
$m$. We visualize formal context using cross tables, as depicted for
the running example \emph{Ben and Jerry's} in~\cref{bjice} (top). A
cross in the table indicates that an object (ice cream flavor) has an
attribute (ice cream ingredient).  A context $\Scon=(H,N,J)$ is called
an \emph{induced sub-context} of $\context$, if
$H\subseteq G, N\subseteq M$ and $I_\Scon=I\cap (H_\Scon \times N)$,
denoted $\Scon \leq \context$. The incidence relation gives rise to
two derivation operators. The first is the derivation of an attribute
$A \subseteq M$ where
$A'=\{g\in G \mid \forall m \in A: (g,m)\in I\}$. The object
derivation $B'$ for $B\subseteq G$ is defined analogously. The
consecutive application of the two derivation operators on an
attribute set (object set) constitutes a \emph{closure operators},
i.e., a idempotent, monotone, and extensive, map. Therefore, the pairs
$(G,'')$ and $(M,'')$ are \emph{closure spaces} with
$\cdot'': \mathcal{P}(G)\to\mathcal{P}(G)$ and
$\cdot'': \mathcal{P}(M)\to\mathcal{P}(M)$. For example,
$\{\text{Dough, Vanilla}\}''=\{\text{Choco, Dough, Vanilla}\}$
in~\cref{bjice}.

A formal concept is a pair
$(A,B) \in \mathcal{P}(G)\times \mathcal{P}(M)$ with $A'=B$ and
$A = B'$, where $A$ is called \emph{extent} and $B$ \emph{intent}. We
denote with $\Ext(\context)$ and $\Int(\context)$ the sets of all
extents and intents, respectively. Each of these sets forms a closure
system associated to the closure operator on the respective base set,
i.e., the object set or the attribute set. Both closure systems are
represented in the \emph{(concept) lattice}
$\BV(\context)=(\mathcal{B}(\context),\subseteq)$, where
$\mathcal{B}(\context)$ denotes the set of all concepts in $\context$
and  for $(A,B), (C,D)\in\mathcal{B}(\context)$ we have $(A,B)\leq
(C,D)\logeq A\subseteq C$.

\subsection{Scales}\label{sec:motivate}
A fundamental problem for the analysis, the computational treatment,
and the visualization of data is the high dimensionality and complex
structure of modern data sets. Hence, the tasks for \emph{scaling}
data sets to a lower number of dimensions and decreasing their
complexity has growing importance. Many unsupervised (machine
learning) procedures were developed and are applied, for example,
multidimensional scaling \cite{MDS,RecentMDS} or principal component
analysis. These scaling methods use non-linear projections of data
objects (points) into a lower dimensional space. While preserving the
notion of object they loose the interpretability of features as well
as the original algebraic object-feature relation. Therefore, the
advantage of explainability when analyzing nominal or ordinal data
cannot be preserved. Furthermore, most scaling approaches require the
representation of the data points in a real coordinate space of some
dimension, which is in turn, already a scaling for many data sets.

A more fundamental approach to scaling, in particular for nominal and
ordinal data, that preserves the interpretable features can be found
in FCA.

\begin{definition}[Scale-Measure (cf. Definition 91, \cite{fca-book})]
\label{def:sm}
Let $\context = (G,M,I)$ and
$\mathbb{S}=(G_{\mathbb{S}},M_{\mathbb{S}},I_{\mathbb{S}})$ be a
formal contexts. The map $\sigma :G \rightarrow G_{\mathbb{S}}$ is
called an \emph{$\mathbb{S}$-measure of $\context$ into the scale
  $\mathbb{S}$} iff the preimage
$\sigma^{-1}(A)\coloneqq \{g\in G\mid \sigma(g)\in A\}$
of every extent $A\in \Ext(\Scon)$ is an extent of $\context$.
\end{definition}

This definition corresponds the notion for \emph{continuity between
  closure spaces} $(G_1,c_1)$ and $(G_2,c_2)$, i.e., a map $f:G_1\to
G_2$ is continuous iff
\begin{equation}
  \label{eq:cont}
\text{for all}\ A\in\mathcal{P}(G_2)$ we have
$c_1(f^{-1}(A))\subseteq f^{-1}(c_2(A)).  
\end{equation}
This property is equivalent to the requirement in~\cref{def:sm} that
the preimage of closed sets is closed, more formally, 
\begin{equation}
  \label{eq:scales}
  \text{for all}\ A\in \mathcal{P}(G_2)\ \text{with}\ c_2(A)=A\ \text{we have}\ 
  f^{-1}(A)=c_1(f^{-1}(A)).  
\end{equation}
Conditions in~(\ref{eq:cont}) and~(\ref{eq:scales}) are
known to be equivalent, since $\eqref{eq:cont}\Rightarrow\eqref{eq:scales}$
follows from
$x\in c_1(f^{-1}(A))\Rightarrow x\in
f^{-1}(c_2(A))\xRightarrow{c_2(A)=A} x\in f^{-1}(A)$. Also, from
$x\in c_1(f^{-1}(A))\Rightarrow x\in
c_1(f^{-1}(c_2(A)))\xRightarrow{\eqref{eq:scales}} x\in
f^{-1}(c_2(A))$ results (\ref{eq:scales})$\Rightarrow$(\ref{eq:cont}).

In the following we may address by $\sigma^{-1}(\Ext(\Scon))$ the set
of all extents of $\context$ that are \emph{reflected} by the scale
context, i.e., $\bigcup_{A\in\Ext(\Scon)}\sigma^{-1}(A)$. Furthermore,
we want to nourish the understanding of scale-measures as consistent
measurements (or views) of the objects in some scale context. In this
sense we understand the map $\sigma$ as an interpretation of the
objects from $\context$ in $\Scon$.

The following corollary can be deduced from the continuity property
above and will be used frequently throughout our work.
\begin{corollary}[Composition Scale-Measures]\label{lem:trans}
  Let $\context$ be a formal context,  $\sigma$ a $\Scon$-measure
  of $\context$ and $\psi$ a $\Tcon$-measure of $\Scon$. Then is
  $\psi \circ \sigma$ a $\Tcon$-measure of $\context$.
\end{corollary}

\begin{figure}[t]
  \label{bjicemeasure}
  \begin{center}
    \hspace{-0.65cm}  \scalebox{0.55}{
      \begin{cxt}
        \cxtName{}
        \att{Brownie (B)}
        \att{Peanut (P)}
        \att{Caramel (Ca)}
        \att{Choco (Ch)}
        \att{Dough (D)}
        \att{Vanilla (V)}
        \obj{x..x..}{Fudge Brownie (FB)}
        \obj{...xxx}{Cookie Dough (CD)}
        \obj{x..xxx}{Half Baked (HB)}
        \obj{..xx..}{Caramel Sutra (CS)}
        \obj{..xx..}{\shortstack{Caramel Chew\\ Chew (CCC)}}
        \obj{.x.x..}{\shortstack{Peanut Butter\\ Cup (PBC)}}
        \obj{x.xx.x}{\shortstack{Salted Caramel\\ Brownie (SCP)}}
      \end{cxt}}  
  \end{center}

 \hspace{-0.5cm} \begin{minipage}{0.53\linewidth}
    \scalebox{0.6}{\colorlet{mivertexcolor}{white}
\colorlet{jivertexcolor}{black}
\colorlet{vertexcolor}{black}
\colorlet{bordercolor}{black}
\colorlet{linecolor}{gray}
\tikzset{vertexbase/.style={semithick, shape=circle, inner sep=2pt, outer sep=0pt, draw=bordercolor,draw opacity=0.4},%
  vertex/.style={vertexbase, fill=vertexcolor!45},%
  mivertex/.style={vertexbase, fill=mivertexcolor!45},%
  jivertex/.style={vertexbase, fill=jivertexcolor!45},%
  divertex/.style={vertexbase, top color=mivertexcolor!45, bottom color=jivertexcolor!45},%
  conn/.style={-, thick, color=linecolor}%
}
\tikzstyle{n} = [text width=2.5cm,align=center]
\begin{tikzpicture}[scale=0.35,font=\footnotesize]
  \begin{scope} 
    \begin{scope} 
      \foreach \nodename/\nodetype/\xpos/\ypos in {%
        0/vertex/1.6624628793523328/5.158005322212052,
        1/vertex/-3.428124546751512/9.18431467431844,
        2/vertex/1.6624628793523328/10.334688774920266,
        3/mivertex/6.7530503054561795/10.334688774920266,
        4/vertex/-3.428124546751512/11.485062875522091,
        5/mivertex/6.7530503054561795/12.635436976123916,
        6/mivertex/-8.518711972855357/13.21062402642483,
        7/vertex/-2.4100070615307425/13.21062402642483,
        8/vertex/12.879872702001563/13.21062402642483,
        9/mivertex/1.6624628793523328/14.360998127026654,
        10/vertex/-3.428124546751512/14.936185177327568,
        11/vertex/6.7530503054561795/15.51137222762848,
        12/mivertex/-3.428124546751512/17.81212042883213,
        13/vertex/-8.518711972855357/18.962494529433954,
        14/mivertex/1.6624628793523328/18.962494529433954,
        15/vertex/-3.428124546751512/22.41361683123943
      } \node[\nodetype] (\nodename) at (\xpos, \ypos) {};
    \end{scope} 
    \begin{scope} 
      \path (3) edge[conn,draw opacity=0.4] (5);
      \path (13) edge[conn,draw opacity=0.4] (15);
      \path (9) edge[conn,draw opacity=0.4] (14);
      \path (6) edge[conn,draw opacity=0.4] (13);
      \path (10) edge[conn,draw opacity=0.4] (14);
      \path (0) edge[conn,draw opacity=0.4] (1);
      \path (14) edge[conn,draw opacity=0.4] (15);
      \path (5) edge[conn,draw opacity=0.4] (11);
      \path (4) edge[conn,draw opacity=0.4] (10);
      \path (11) edge[conn,draw opacity=0.4] (14);
      \path (9) edge[conn,draw opacity=0.4] (12);
      \path (7) edge[conn,draw opacity=0.4] (13);
      \path (1) edge[conn,draw opacity=0.4] (6);
      \path (8) edge[conn,draw opacity=0.4] (14);
      \path (0) edge[conn,draw opacity=0.4] (2);
      \path (3) edge[conn,draw opacity=0.4] (9);
      \path (2) edge[conn,draw opacity=0.4] (11);
      \path (1) edge[conn,draw opacity=0.4] (9);
      \path (0) edge[conn,draw opacity=0.4] (8);
      \path (6) edge[conn,draw opacity=0.4] (12);
      \path (2) edge[conn,draw opacity=0.4] (7);
      \path (12) edge[conn,draw opacity=0.4] (15);
      \path (7) edge[conn,draw opacity=0.4] (10);
      \path (1) edge[conn,draw opacity=0.4] (4);
      \path (0) edge[conn,draw opacity=0.4] (3);
      \path (1) edge[conn,draw opacity=0.4] (7);
      %
    \end{scope}
    \begin{scope} 
      \foreach \nodename/\labelpos/\labelopts/\labelcontent in {%
        1/below/n/{HB},    
        2/below/n/{SCB},   
        4/below/n/{CD},    
        4/above/n/{D},     
        8/below/n/{PBC},    
        8/above/n/{PB}, 
        10/above/n/{VI},    
        11/above/n/{Ca},    
        11/below/n/{CCC,CS},
        13/above/n/{B}, 
        13/below/n/{CFB},
        15/above/n/{C}
      } \coordinate[label={[\labelopts]\labelpos:{\labelcontent}}](c) at (\nodename);
    \end{scope}
  \end{scope}
\end{tikzpicture}}
  \end{minipage}$\Rightarrow$\hspace{-0.5cm}
  \begin{minipage}{0.4\linewidth}
    \scalebox{0.7}{\colorlet{mivertexcolor}{black!80}
\colorlet{jivertexcolor}{black!80} 
\colorlet{vertexcolor}{black!80}
\colorlet{bordercolor}{black!80}
\colorlet{linecolor}{gray}
\tikzset{vertexbase/.style={semithick, shape=circle, inner sep=2pt, outer sep=0pt, draw=bordercolor},%
  vertex/.style={vertexbase, fill=vertexcolor!45},%
  mivertex/.style={vertexbase, fill=mivertexcolor!45},%
  jivertex/.style={vertexbase, fill=jivertexcolor!45},%
  divertex/.style={vertexbase, top color=mivertexcolor!45, bottom color=jivertexcolor!45},%
  conn/.style={-, thick, color=linecolor}%
}
\begin{tikzpicture}[scale=0.6]
  \begin{scope} 
    \begin{scope} 
      \foreach \nodename/\nodetype/\xpos/\ypos in {%
        0/vertex/-1.8748824082784559/5.991815616180622,
        1/jivertex/-5.658835428228559/9.081132287699923,
        2/jivertex/-2.005268109125117/9.121072436500471,
        3/vertex/-3.7375352775164608/11.337629350893696,
        4/divertex/-7.6118532455315115/11.374882408278458,
        5/divertex/4.465366545453902/12.685526054254581,
        6/divertex/0.7/12.874788334901222,
        7/divertex/-2.1307619943555956/13.28692380056444,
        8/mivertex/-5.413922859830667/13.665945437441204,
        9/vertex/-2.7130761994355588/16.590310442144872
      } \node[\nodetype] (\nodename) at (\xpos, \ypos) {};
    \end{scope}
    \begin{scope} 
      \path (7) edge[conn] (9);
      \path (1) edge[conn] (4);
      \path (0) edge[conn] (5);
      \path (3) edge[conn] (7);
      \path (0) edge[conn] (1);
      \path (6) edge[conn] (9);
      \path (1) edge[conn] (3);
      \path (8) edge[conn] (9);
      \path (0) edge[conn] (2);
      \path (5) edge[conn] (9);
      \path (2) edge[conn] (6);
      \path (3) edge[conn] (8);
      \path (2) edge[conn] (3);
      \path (4) edge[conn] (8);
    \end{scope}
    \begin{scope} 
      \foreach \nodename/\labelpos/\labelopts/\labelcontent in {%
        1/below//{HB},
        2/below//{SCB},
        4/below//{CD},
        4/above//{D},
        5/below//{PBC},
        5/above//{P},
        6/below//{CS,CCC},
        6/above//{Ca},
        7/below//{FB},
        7/above//{B},
        8/above//{V},
        9/above//{Ch}
      } \coordinate[label={[\labelopts]\labelpos:{\labelcontent}}](c) at (\nodename);
    \end{scope}
  \end{scope}
\end{tikzpicture}}
  \end{minipage}
  \caption{A scale context (top), its concept lattice (bottom right)
    for which $\id_G$ is a scale-measure of the context in
    \cref{bjice} and the reflected extents
    $\sigma^{-1}(\Ext(\context))$ (bottem left) indicated as
    non-transparent.}
\end{figure}
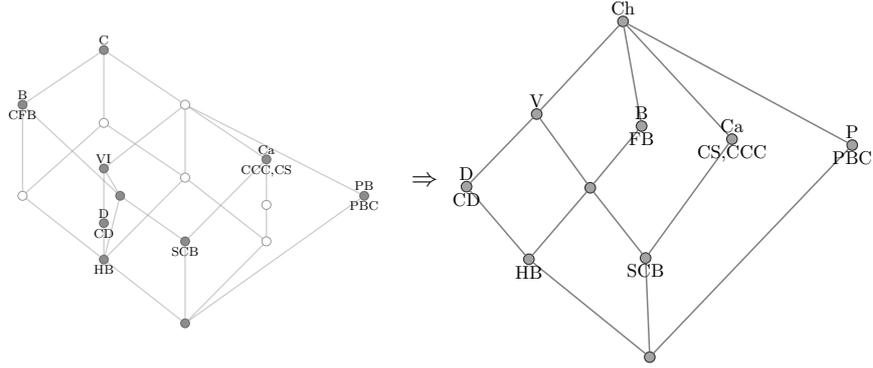

In \cref{bjicemeasure} we depict a scale-measure and its concept
lattice for our running example context \emph{Ben and Jerry's}
$\context_{\text{BJ}}$, cf.~\cref{bjice}. This scale-measure uses the same
object set as the original context and maps every object to
itself. The attribute set is comprised of six elements, which may
reflect the taste, instead of the original nine attributes that
indicated the used ingredients. The specified scale-measure map allows
for a human comprehensible interpretation of $\sigma^{-1}$, as
indicated by the grey colored concepts in~\cref{bjicemeasure}
(bottom). In this figure we observe that the concept lattice of the
scale-measure reflects ten out of the sixteen concepts in
$\mathfrak{B}(\context_{\text{BJ}})$. 

The empirical observations about the afore presented example
scale-measure for some context $\context$ lead
to the question whether scale-measures are always at least as
comprehensible as the context $\context$ itself.
A typical (objective) measure for the complexity of lattices is given
by the following quantity.

\begin{definition}[Order Dimension (cf. Definition
  82,~\cite{fca-book})]\label{def:orddim}
  An ordered set $(P,\leq)$ has \emph{order dimension}
  $\dim(P,\leq)=n$ iff it can be embedded in a direct product of $n$
  chains and $n$ is the smallest number for which this is possible.
\end{definition}

The order dimension of $\BV(\context_{\text{BJ}})$ is three whereas the
concept lattice of the given scale-measure is two. Finding low
dimensional scale-measures for large and complex data sets is a
natural approach towards comprehensible data analysis, as demonstrated
in ~\cref{prop:dim}. In particular, we will answer the question if the
order dimension of scale-measures is bound by the order dimension of
$\BV(\context)$.

Another notion for comparing scale-measures is provided by a natural
order relation amongst scales~\cite[Definition 92]{fca-book}). We may
present in the following a more general definition within the scope of
scale-measures. 


\begin{definition}[Scale-Measure Refinement]\label{def:sm-refine}
  Let the set of all scale-measures of a context be denoted by
  $\Sh(\context)\coloneqq \{(\sigma, \Scon)\mid \sigma$ is a
  $\Scon-$measure of $\context \}$. For
  $(\sigma,\Scon),(\psi,\Tcon)\in \Sh(\context)$ we say
  $(\sigma,\Scon)$ is a \emph{coarser} scale-measure of $\context$
  than $(\psi,\Tcon)$, iff
  $\sigma^{-1}(\Ext(\Scon)) {\subseteq}
  \psi^{-1}(\Ext(\Tcon)$. Analogously we then say $(\psi,\Tcon)$ is \emph{finer} than $(\sigma,\Scon)$.
  If $(\sigma,\Scon)$ is finer and coarser than $(\psi,\Tcon)$ we call
  them \emph{equivalent scale-measures}.
\end{definition}

We remark that the finer relation as well as coarser relation
constitute (partial) order relations on the set of all scale-measure
for context $\context$, since they are obviously reflexive,
anti-symmetric, and the transitivity follow from the continuity of the
composition of scale maps. Hence, we may refer to the refinement
(order) using the symbol $\leq$. By computing
scale-measures with coarser scale contexts with respect to the
\emph{refinement order} we can provide a more general conceptual view
on a data set. The study of such views, e.g. the ice cream tastes in
our running example presented in~\cref{bjicemeasure}, is in a similar
fashion to the Online Analytical Processing tools for multidimensional
databases.

Moreover, the set of all scale-measure for some formal context enables
an abstract analytical structure to navigate and explore a data set
with. Yet, despite the supposed usefulness of the scale-measures, there are up
until now no existing methods, to the best of our knowledge, for the
generation and evaluation of scale-measures, in particular with
respect to data science applications. 

Both tasks, the generation and the evaluation of scale-measures, will
be tackled in the next section using a novel navigation approach among
them.

\section{Navigation though Conceptual Measurement}\label{sec:methods}
Based on the just introduced refinement order of scale-measures we
provide in this section the means for efficiently browsing this
structure. Given a data set, the presented methods are able to compute
arbitrary scale abstractions and the structure operations that connect
them, which resembles a \emph{navigation} through conceptual
measurements. To lay the foundation for the navigation methods we
start with analyzing the structure of all scale-measures. Thereafter
we will present a thorough description of the navigation problem and
its solution.

\begin{lemma}
  The scale-measure equivalence is an equivalence relation on the set
  of scale-measures.
\end{lemma}
\begin{proof}
  Let $(\sigma,\Scon),(\psi,\Tcon),(\phi,\Ocon)\in \Sh(\context)$ be
  scale-measures of context $\context$. Using ~\cref{def:sm-refine} we
  know from $(\sigma,\Scon)\sim (\psi,\Tcon)$ that
  $\sigma^{-1}(\Ext(\Scon)) = \psi^{-1}(\Ext(\Tcon))$, from which the
  reflexivity and the symmetry of $\sim$ can be inferred. Analogously
  we can infer for $(\sigma,\Scon)\sim(\psi,\Tcon)$ and
  $(\psi,\Tcon)\sim(\phi,\Ocon)$ that
  $(\sigma,\Scon)\sim(\phi,\Ocon)$.\qed
\end{proof}

Note that for two given equivalent scale-measures that their
scale-measure equivalence does not imply the existence of an bijective
scale-measure between them. Yet, a minor requirement to the
scale-measure map leads to a useful link.

\begin{lemma}\label{lem:eqiso}
  Let $(\sigma,\Scon),(\psi, \Tcon)\in \Sh(\context)$ with
  $(\sigma,\Scon)\sim (\psi, \Tcon)$ and $\sigma,\psi$ are surjective
  maps. Then $\sigma^{-1}\circ \psi$ is an order isomorphism from
  $(\Ext(\Scon),\subseteq)$ to $(\Ext(\Tcon),\subseteq)$.
\end{lemma}
\begin{proof}
  From~\cite[Proposition 118]{fca-book} we have that $\sigma^{-1}$ is
  a injective $\wedge$-preserving order embedding of
  $(\Ext(\Scon),\subseteq)$ into $(\Ext(\context),\subseteq)$ and
  thereby a bijective $\wedge$-preserving order embedding into
  $(\sigma^{-1}(\Ext(\Scon)),\subseteq)$. The analogue holds for
  $\psi^{-1}$ from $\Ext(\Tcon)$ into $\psi^{-1}(\Ext(\Tcon))$. Due to
  $(\sigma,\Scon)\sim (\psi, \Tcon)$ we know that
  $\sigma^{-1}(\Ext(\Scon))=\psi^{-1}(\Ext(\Tcon))$, which results in
  $\sigma^{-1}$ being a bijective $\wedge$-preserving order embedding
  into $\psi^{-1}(\Ext(\Tcon))$. Hence, when restricting
  $\sigma^{-1}\circ \psi:\mathcal{P}(G_\mathbb{S})\to
  \mathcal{P}(G_{\mathbb{T}})$ to the respective extent set we obtain
  a bijective map. The fact that all formal contexts are finite
  (throughout this work) and the monotonicity of the lifts of
  $\sigma^{-1}$ and $\psi$ to their respective power sets imply the
  required order preserving property follow.\qed
\end{proof}
We may stress that the required surjectivity is not constraining the
application of scale-measures, since any object $g$ of a scale-context
having an empty preimage may just be removed from the scale-context
without consequences to the analysis. 

The just discussed equivalence relation together with the refinement
order allows to cope with the set of all scale-measures
$\Sh(\context)$ in a meaningful way.

\begin{definition}[Scale-Hierarchy]\label{def:Sh}
  Given a formal context $\context$ and its set of all scale-measures
  $\Sh(\context)$, we call
  $\SH(\context)\coloneqq(\nicefrac{\Sh(\context)}{\sim},\leq)$ the
  \emph{scale-hierarchy} of $\context$.
\end{definition}

The order structure thus given represents all possible means of
scaling a (contextual) data set. Yet, it seems hardly comprehensible
or even applicable in that form. Therefore the goal for the rest of
this section is to achieve a characterization of said structure in
terms of closure systems. 

\begin{lemma}\label{lem:csctx}
  Let $G$ be a set and $\mathcal{A}\subseteq \mathcal{P}(G)$ be a
  closure system. Furthermore, let
  $\context_{\mathcal{A}}=(G,\mathcal{A},\in)$ be a formal context
  using the element relation as incidence. Then the set of extents
  $\Ext(\context_{\mathcal{A}})$ is equal to the closure system
  $\mathcal{A}$.
\end{lemma}
\begin{proof}
  For any set $D\subseteq G$ and $A\in\mathcal{A}$ we find ($\ast$) $D\subseteq
  A\implies A\in D'$. Since $\mathcal{A}$ is a closure system and
  $D''=\bigcap D'$ we see that $D''\in\mathcal{A}$, hence,
  $\Ext(\context_{\mathcal{A}})\subseteq \mathcal{A}$. Conversely, for $A\in
  \mathcal{A}$ we can draw from ($\ast$) that $A''=A$, thus
  $A\in\Ext(\context_{\mathcal{A}})$.\qed
\end{proof}

We want to further motivate the constructed formal context
$\context_{\mathcal{A}}$ and its particular utility with respect to
scale-measures for some context $\context$. Since both contexts have
the same set of objects, we may study the use of the identity map $\id:G\to
G, g\mapsto g$ as scale-measure map. 

\begin{lemma}[Canonical Construction]\label{lem:cssm}
  For a context $\context$ and any $\Scon$-measure $\sigma$ is $\id$ a
  $\context_{\sigma^{-1}(\Ext(\Scon))}$-measure of $\context$, i.e.,
  $(\id,\context_{\sigma^{-1}(\Ext(\Scon))})\in\Sh(\context)$.
\end{lemma}
\begin{proof}
  \cref{lem:csctx} gives that
  $\Ext(\context_{\sigma^{-1}(\Ext(\Scon))})$ is equal to
  $\sigma^{-1}(\Ext(\Scon))$. Since $(\sigma,\Scon)\in\Sh(\context)$,
  i.e., $(\sigma,\Scon)$ is a scale-measure of $\context$, 
  we see that the preimage $\sigma^{-1}(\Ext(\Scon))\subseteq \Ext(\context)$, and
  thus
  $\id^{-1}(\Ext(\context_{\sigma^{-1}(\Ext(\Scon))}))\subseteq\Ext(\context)$.
  \qed
\end{proof}

Using the canonical construction of a scale-measure, as given above,
we can facilitate the understanding of the scale-hierarchy
$\SH(\context)$.

\begin{proposition}[Canonical Representation]\label{prop:eqi-scale}
  Let $\context = (G,M,I)$ be a formal context with scale-measure $(\Scon,\sigma)\in
  \Sh(\context)$, then $(\sigma,\Scon)\sim (\id, \context_{\sigma^{-1}(\Ext(\Scon))})$.
\end{proposition}
\begin{proof}
  \cref{lem:cssm} states that $\id$ is a
  $\context_{\sigma^{-1}(\Ext(\Scon))}$-measure of
  $\context$. Furthermore, from~\cref{lem:csctx} we know that the
  extent set of $\context_{\sigma^{-1}(\Ext(\Scon))}$ is
  $\sigma^{-1}(\Ext(\Scon))$, as required
  by~\cref{def:sm-refine}.\qed
\end{proof}

Equipped with this proposition we are now able to compare sets of
scale-measures for a given formal context $\context$ solely based on
their respective attribute sets in the canonical
representation. Furthermore, since these representation sets are
sub-closure systems of $\Ext(\context)$, by~\cref{def:sm}, we may
reformulate the problem for navigating scale-measures using
sub-closure systems and their relations. For this we want to nourish
the understanding of the correspondence of scale-measures and
sub-closure systems in the following.

\begin{proposition}\label{cor:size}
  For a formal context $\context$ and the set of all sub-closure
  systems $\mathfrak{C}(\context)
  \subseteq\mathcal{P}(\Ext(\context))$ together with the inclusion
  order,  the following map is an order isomorphism:
  \[{i}:\mathfrak{C}(\context)\to\Sh(\context)_{/\sim},\ \mathcal{A}\mapsto
    i(\mathcal{A})\coloneqq(\id,\context_{\mathcal{A}})\]
\end{proposition}
\begin{proof}
  Let $\mathcal{A},\mathcal{B}\subseteq \Ext(\context)$ be two closure
  systems on $G$. Then the images of $\mathcal{A}$ respectively
  $\mathcal{B}$ under $i$ are a scale-measures of $\context$,
  according to \cref{lem:cssm}, with extents $\mathcal{A}$ and
  $\mathcal{B}$, respectively. Since $\mathcal{A}\neq
  \mathcal{B}\iff\Ext(\context_{\mathcal{A}})\neq\Ext(\context_{\mathcal{B}})$
  are different and therefore $(\id,\context_{\mathcal{B}})\not\sim
  (\id,\context_{\mathcal{B}})$, thus, $i$ is an injective map. For
  the surjectivity of $i$ let
  $[(\sigma,\Scon)]\in\Sh(\context)_{/\sim}$, then
  $(\id,\context_{\sigma^{-1}(\Ext(\Scon))})\sim(\sigma,\Scon)$, i.e.,
  an equivalent representation having extents
  $\sigma^{-1}(\Ext(\Scon))\subseteq\Ext(\context)$ and
  $i(\sigma^{-1}(\Ext(\Scon)))=(\id,\context_{\sigma^{-1}(\Ext(\Scon))})$. Finally,
  for $\mathcal{A}\subseteq \mathcal{B}$ we find that
  $i(\mathcal{A})\subseteq i(\mathcal{B})$, since
  $\Ext(\context_{\mathcal{A}}) \subseteq
  \Ext(\context_{\mathcal{B}})$, as required.\qed
\end{proof}

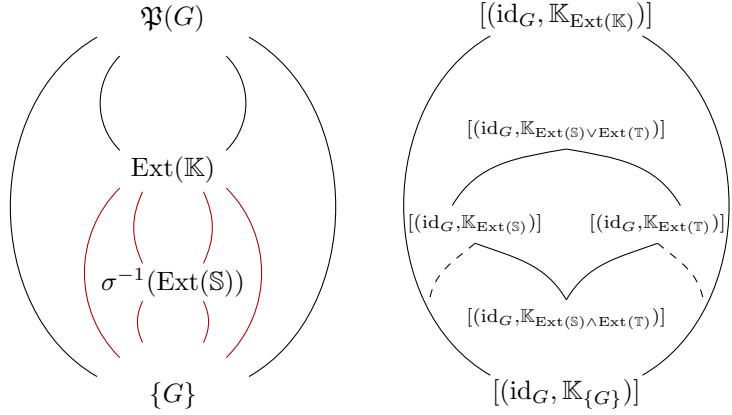
\begin{figure}[t]
  \centering
  \begin{tikzpicture}

    \draw (1,-4.25) to[out=30, in=-30] (1,0.25);
    \draw (-1,-4.25) to[out=150, in=-150] (-1,0.25);

    \draw[draw=black!30!red] (0.7,-4) to[out=45, in=-45] (0.7,-1.75);
    \draw[draw=black!30!red] (-0.7,-4) to[out=135, in=-135] (-0.7,-1.75);

    \draw (0.7,-1.25) to[out=45, in=-45] (0.7,0);
    \draw (-0.7,-1.25) to[out=135, in=-135] (-0.7,0);

    \draw[draw=black!30!red] (0.4,-3.8) to[out=65, in=-65] (0.4,-3.25);
    \draw[draw=black!30!red] (-0.4,-3.8) to[out=115, in=-115] (-0.4,-3.25);

    \draw[draw=black!30!red] (0.4,-2.75) to[out=65, in=-65] (0.4,-1.8);
    \draw[draw=black!30!red] (-0.4,-2.75) to[out=115, in=-115] (-0.4,-1.8);

    \node[draw opacity = 0,draw=white, text=black] at (0,0.5)
    {$\mathfrak{P}(G)$};
    \node[draw opacity = 0,draw=white, text=black] at (0,-4.5)
    {$\{G\}$};
    \node[draw opacity = 0,draw=white, text=black] at (0,-1.5)
    {$\Ext(\context)$};
    \node[draw opacity = 0,draw=white, text=black] at (0,-3)
    {$\sigma^{-1}(\Ext(\Scon))$}; 
  \end{tikzpicture}
  \begin{tikzpicture}

    \draw (1,-4.25) to[out=30, in=-30] (1,0.25);
    \draw (-1,-4.25) to[out=150, in=-150] (-1,0.25);

    \draw (1.5,-2) to[out=120, in=-10] (0,-1.25);
    \draw (-1.5,-2) to[out=60, in=-170] (-0,-1.25);

    \draw (-1.2,-2.5) to[out=-20, in=120] (-0,-3.25);
    \draw[dashed] (-1.2,-2.5) to[out=220, in=80] (-1.8,-3.25);

    \draw (1.2,-2.5) to[out=200, in=60] (-0,-3.25);
    \draw[dashed] (1.2,-2.5) to[out=-40, in=100] (1.8,-3.25);
    \node[draw opacity = 0,draw=white, text=black] at (0,0.5)
    {$[(\id_{G},\context_{\Ext(\context)})]$};
    \node[draw opacity = 0,draw=white, text=black] at (0,-4.5)
    {$[(\id_{G}, \context_{\{G\}})]$};

    \node[draw opacity = 0,draw=white, text=black] at (-1.2,-2.25)
    {$\scriptstyle[(\id_G,\context_{\Ext(\Scon)})]$};
    \node[draw opacity = 0,draw=white, text=black] at (1.2,-2.25)
    {$\scriptstyle[(\id_G,\context_{\Ext(\Tcon)})]$};

    \node[draw opacity = 0,draw=white, text=black] at (0,-3.5)
    {$\scriptstyle[(\id_G,\context_{\Ext(\Scon)\wedge\Ext(\Tcon)})]$};
    \node[draw opacity = 0,draw=white, text=black] at (0,-1)
    {$\scriptstyle[(\id_G, \context_{\Ext(\Scon)\vee\Ext(\Tcon)})]$};

  \end{tikzpicture}
  \caption{Scale-Hierarchy of $\context$ (right) and embedded in boolean
    $\mathbb{B}_G$}
  \label{fig:SmAsCl}
\end{figure}

This order isomorphism allows us to analyze the structure of the
scale-hierarchy by studying the related closure systems. For instance,
the problem of computing $|\SH(\context)|$, i.e., the size of the
scale-hierarchy. In the case of the boolean context
$\context_{\mathcal{P}(G)}$ this problem equivalent to the question
for the number of Moore families, i.e., the number of closure systems
on $G$. This number grows tremendously in $|G|$ and is known up to
$|G|=7$, for which it is known~\cite{size1,size2,size3} to be
$14\,087\,648\,235\,707\,352\,472$. In the general case the size of
the scale-hierarchy is equal to the size of the order ideal
$\downarrow\Ext(\context)$ in
$\mathfrak{C}(\context_{\mathcal{P}(G)})$.

The fact that the set of all closure systems on $G$ is again a closure
system~\cite{CASPARD2003241}, which is lattice ordered by set
inclusion, allows for the following statement.

\begin{corollary}[Scale-hierarchy Order]
  For a formal context $\context$, the scale-hierarchy $\SH(\context)$
  is lattice ordered.
\end{corollary}

We depicted this lattice order relation in the form of abstract
visualizations in~\cref{fig:SmAsCl}. In the bottom (right) we see the
most simple scale which has only one attribute, $G$. The top (right)
element in this figure is then the scale which has all extents of
$\context$. On the left we see the lattice ordered set of all closure
systems on a set $G$, in which we find the embedding of the hierarchy
of scales.

\begin{proposition}\label{prop:lattice}
        Let $(\sigma,\Scon),(\psi,\Tcon)\in \SH(\context)$ and let
        $\wedge,\vee$ be the natural lattice operations in
        $\SH(\context)$, (induced by the lattice order relation). We
        then find that:
    \begin{itemize}
    \item[Meet]: $(\sigma,\Scon)\wedge (\psi,\Tcon)=
      (\id,\context_{\sigma^{-1}(\Ext(\Scon))\cap \psi^{-1}(\Ext(\Tcon ))})$,
    \item[Join]: $(\sigma,\Scon)\vee (\psi,\Tcon)=
      (\id,\context_{\{A \cap B \mid A \in \sigma^{-1}(\Ext(\Scon)), B
        \in \psi^{-1}(\Ext(\Tcon))\}})$.
  \end{itemize}
\end{proposition}
\begin{proof}
  \begin{inparaenum}
  \item For the preimages $i^{-1}(\sigma,\Scon)$,
    $i^{-1}(\psi,\Tcon)$ (\cref{cor:size}) we can compute their
    meet~\cite{CASPARD2003241}, which yields \[i^{-1}(\sigma,\Scon)\wedge
    i^{-1}(\psi,\Tcon)=\sigma^{-1}(\Ext(\Scon))
    \cap \psi^{-1}(\Ext(\Tcon)).\]
  
\item The join~\cite{CASPARD2003241} of the scale-measure preimages
  under $i$ (\cref{cor:size}) is equal to $\{A \cap B \mid A \in
  \sigma^{-1}(\Ext(\Scon)), B \in \psi^{-1}(\Ext(\Tcon))\}$ , which
  results in the required expression by applying the order isomorphism
  $i$. 
  \end{inparaenum}\qed
\end{proof}

\subsection{Propositional Navigation through Scale-Measures} 
Although the canonical representation of scale-measures is complete up
to equivalence~\cref{prop:eqi-scale}, this representation eludes human
explanation to some degree. In particular the use of the extentional
structure of $\context$ as attributes provides insight to the
scale-hierarchy itself, however, not to the data, i.e., the objects,
attributes, and their relation. A formulation of scales using
attributes from $\context$, and their combinations, seems more natural
and more comprehensible. For this, we employ an approach as used in
\cite{logiscale}. In their work the authors used a logic on the
context's attributes to introduce new attributes. The advantage is
that the so newly introduced attributes have a real-world semantic in
terms of the measured properties. In this work we use propositional
logic, which leads to the following problem description.

\begin{problem}[Navigation Problem]\label{problem:navi}
  For a formal context $\context$, a scale-measure $(\sigma,\Scon)\in
  \Sh(\context)$ and
  $M_\Tcon\subseteq\mathcal{L}(M,\{\wedge,\vee,\neg\})$, compute an
  equivalent scale-measure $(\psi,\Tcon)\in\Sh(\context)$, i.e.,
  $(\sigma,\Scon)\sim(\psi,\Tcon)$, where $(h,m)\in
  I_\Tcon\Leftrightarrow \psi^{-1}(h)^{I}\models m$.
\end{problem}

The attributes of $\Tcon$ are logical expression build from the
attributes of $\context$, and are thus interpretable in terms of the
measurements by the attributes $M$ from $\context$. For example, we can
express the \emph{Choco} taste attribute of our running example
(\cref{bjicemeasure}) as the disjunction of the ingredients
\emph{Choco Ice} or \emph{Choco Pieces}, i.e.  \emph{Choco}$\coloneqq
$\emph{Choco Ice}$\vee$\emph{Choco Pieces}. For any scale-measure
$(\sigma,\Scon)$, such an equivalent scale-measure, as searched for
in~\cref{problem:navi}, is not necessarily unique, and the problem
statement does not favor any of the possible solutions.

To understand the semantics of the logical operations, we first
investigate their contextual derivations. For $\phi \in
\mathcal{L}(M,\{\wedge,\vee,\neg\})$ we let $\Var(\phi)$ be the set of
all propositional variables in the expression $\phi$. We require from
$\phi\in\mathcal{L}(M,\{\wedge,\vee,\neg\})$ that $|\Var(\phi)|>0$. 

\begin{lemma}[Logical Derivations]\label{lem:deri}
  Let $\context=(G,M,I)$ be a formal context,
  $\phi_\wedge\in \mathcal{L}(M,\{\wedge\})$,
  $\phi_\vee\in\mathcal{L}(M,\{\vee\})$, $\phi_\neg \in
  \mathcal{L}(M,\{\neg\})$, with scale contexts
  $(G,\{\phi\},I_{\phi})$ having the incidence $(g,\phi)\in
  I_{\phi}\iff g^{I}\models \phi$ for
  $\phi\in\{\phi_\vee,\phi_\wedge,\phi_\neg\}$. Then we find
  \begin{enumerate}[i)]
  \item $\{\phi_\wedge\}^{I_{\phi_{\wedge}}} = \Var(\phi_{\wedge})^{I}$, 
  \item $\{\phi_\vee\}^{I_{\phi_{\vee}}}=\bigcup_{m\in
      \Var(\phi_{\vee})}\{m\}^{I}$,
  \item $\{\phi_\neg\}^{I_{\phi_{\neg}}} = G\setminus \{n\}^{I}$ with
    $\phi_\neg = \neg n$ for $n\in M$.
  \end{enumerate}
\end{lemma}
\begin{proof}
  \begin{inparaenum}[i)]
  \item For $g\in G$ if $gI_{\phi_\wedge}\phi_{\wedge}$, then
    $\{g\}^I\models \phi_\wedge$ and thereby $\Var(\phi_{\wedge})\subseteq
    \{g\}^{I}$. Hence $g\in \Var(\phi_{\wedge})^{I}$. In case
    $(g,\phi_{\wedge})\not\in I_{\phi_\wedge}$ it holds that $\Var(\phi_{\wedge})\not\subseteq
    \{g\}^{I}$ and thereby $g\not\in \Var(\phi_{\wedge})^{I}$.
  \item For $g\in G$ if $gI_{\phi_\vee}\phi_\vee$ we have
    $\{g\}^I\models \phi_{\vee}$. Hence, $\exists m\in
    \Var(\phi_{\vee})$ with $g\in m^{I}$ and therefore $g$ is in the
    union. If $(g,\phi_{\vee})\not\in I_{\phi_\vee}$ there does not exists such
    a $m\in \Var(\phi_{\wedge})$ and $g\not\in\bigcup_{m\in \Var(\phi)}m^{I}$.
  \item For any $n\in M$ we have $\phi_\neg = \neg n$. Hence, for
    $g\in G$ if $gI_{\phi_\neg}\phi_\neg$ we find $g\not\in \{n\}^{I}$.
    Conversely, if $(g,\phi_\neg)\not\in I_{\phi_\neg}$ it follows that $g\in \{n\}^{I}$.
  \end{inparaenum}\qed
\end{proof}

Naturally, the results from the lemma above generalize to scale
contexts with more than one logical expression in the set of
attributes. How this is done is demonstrated
in~\cref{sec:apos}. Moreover, more complex formulas, i.e.,
$\phi\in\mathcal{L}(M,\{\wedge,\vee,\neg\})$, can be recursively
deconstructed and then treated with \cref{lem:deri}.  In particular,
with respect to unsupervised machine learning, we may mention the
connection to the task of clustering attributes, as studied by Kwuida
et al.~\cite{leonardOps}.

\begin{proposition}[Logical Scale-Measure]\label{prop:logiattr}
  Let $\context$ be a formal context and let $\phi \in
  \mathcal{L}(M,\{\wedge, \vee, \neg\})$, then $\id_G$ is a $(G,
  \{\phi\},I_\phi)$-measure of $\context$ iff
  $\{\phi\}^{I_{\phi}}\in \Ext(\context)$.
\end{proposition}
\begin{proof}
  Since $|\{\phi\}|=1$ we find that $(G, \{\phi\},I_\phi)$ has at
  least one and most two possible extents, $\{\{\phi\}^{I_{\phi}},
  G\}$. If the map $\id_{G}$ is a scale-measure of $\context$, then
  $\id_G^{-1}(\{\phi\}^{I_{\phi}}) = \{\phi\}^{I_{\phi}} \in
  \Ext(\context)$.  Conversely, if $ \{\phi\}^{I_{\phi}} \in
  \Ext(\context)$ so is $\id_G^{-1}(\{\phi\}^{I_{\phi}})$, hence,
  $\id_{G}$ is $(G, \{\phi\},I_\phi)$-measure of $\context$. \qed
\end{proof}

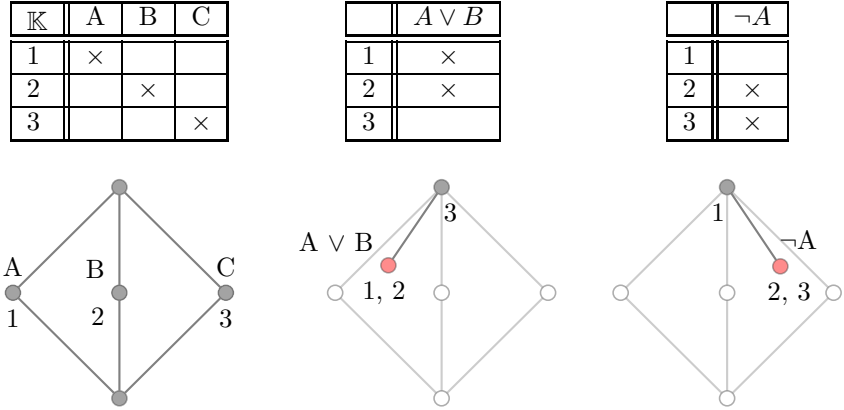
\begin{figure}[t]\label{fig:concl}
  \begin{minipage}{0.32\linewidth}  \centering
  \begin{cxt} 
    \cxtName{$\context$}
    \att{A}
    \att{B}
    \att{C}
    \obj{x..}{1}
    \obj{.x.}{2}
    \obj{..x}{3}
  \end{cxt}    
  \end{minipage}
  \begin{minipage}{0.32\linewidth}  \centering
  \begin{cxt}
    \cxtName{}
    \att{$A\vee B$}
    \obj{x}{1}
    \obj{x}{2}
    \obj{.}{3}
  \end{cxt}    
  \end{minipage}
  \begin{minipage}{0.32\linewidth}  \centering
  \begin{cxt}
    \cxtName{}
    \att{$\neg A$}
    \obj{.}{1}
    \obj{x}{2}
    \obj{x}{3}
  \end{cxt}
  \end{minipage}\\
\ \\
\ \\
  \begin{minipage}{0.32\linewidth}  \centering
    \colorlet{mivertexcolor}{black!80}
\colorlet{jivertexcolor}{black!80}
\colorlet{vertexcolor}{black!80}
\colorlet{bordercolor}{black!80}
\colorlet{linecolor}{gray}
\tikzset{vertexbase/.style={semithick, shape=circle, inner sep=2pt, outer sep=0pt, draw=bordercolor,draw opacity=0.4},%
  vertex/.style={vertexbase, fill=vertexcolor!45},%
  mivertex/.style={vertexbase, fill=mivertexcolor!45},%
  jivertex/.style={vertexbase, fill=jivertexcolor!45},%
  divertex/.style={vertexbase, top color=mivertexcolor!45, bottom color=jivertexcolor!45},%
  conn/.style={-, thick, color=linecolor}%
}
\tikzstyle{s} = [label distance=0.1cm]
\begin{tikzpicture}[scale=1.4]
  \begin{scope} 
    \begin{scope} 
      \foreach \nodename/\nodetype/\xpos/\ypos in {%
        0/vertex/1/0.0,
        1/divertex/0.0/1,
        2/divertex/1/1,
        3/divertex/2/1,
        4/vertex/1/2
      } \node[\nodetype] (\nodename) at (\xpos, \ypos) {};
    \end{scope}
    \begin{scope} 
      \path (1) edge[conn] (4);
      \path (3) edge[conn] (4);
      \path (0) edge[conn] (3);
      \path (2) edge[conn] (4);
      \path (0) edge[conn] (2);
      \path (0) edge[conn] (1);
    \end{scope}
    \begin{scope} 
      \foreach \nodename/\labelpos/\labelopts/\labelcontent in {%
        1/below/s/{1},
        1/above/s/{A},
        2/below left/s/{2},
        2/above left/s/{B},
        3/below/s/{3},
        3/above/s/{C}
      } \coordinate[label={[\labelopts]\labelpos:{\labelcontent}}](c) at (\nodename);
    \end{scope}
  \end{scope}
\end{tikzpicture}    
  \end{minipage}
  \begin{minipage}{0.32\linewidth}  \centering
    \colorlet{mivertexcolor}{red}
\colorlet{jivertexcolor}{black!80}
\colorlet{vertexcolor}{black!80}
\colorlet{bordercolor}{black!80}
\colorlet{linecolor}{gray}
\tikzset{vertexbase/.style={semithick, shape=circle, inner sep=2pt, outer sep=0pt, draw=bordercolor,draw opacity=0.4},%
  vertex/.style={vertexbase, fill=vertexcolor!45},%
  mivertex/.style={vertexbase, fill=mivertexcolor!45},%
  jivertex/.style={vertexbase, fill=jivertexcolor!45},%
  divertex/.style={vertexbase, top color=white, bottom color=white},%
  conn/.style={-, thick, color=linecolor}%
}
\tikzstyle{s} = [label distance=0.1cm]
\begin{tikzpicture}[scale=1.4]
  \begin{scope} 
    \begin{scope} 
      \foreach \nodename/\nodetype/\xpos/\ypos in {%
        0/mivertex/0.5/1.26,
        1/jivertex/1/2,
        2/divertex/1/0.0,
        3/divertex/0.0/1,
        4/divertex/1/1,
        5/divertex/2/1
      } \node[\nodetype] (\nodename) at (\xpos, \ypos) {};
    \end{scope}
    \begin{scope} 
      \path (0) edge[conn] (1);
      \path (3) edge[conn,draw opacity=0.4] (1);
      \path (5) edge[conn,draw opacity=0.4] (1);
      \path (2) edge[conn,draw opacity=0.4] (5);
      \path (4) edge[conn,draw opacity=0.4] (1);
      \path (2) edge[conn,draw opacity=0.4] (4);
      \path (2) edge[conn,draw opacity=0.4] (3);
    \end{scope}
    \begin{scope} 
      \foreach \nodename/\labelpos/\labelopts/\labelcontent in {%
        0/below/s/{1, 2\ \ },
        0/above left/s/{A $\vee$ B},
        1/below/s/{\ \ 3}
      } \coordinate[label={[\labelopts]\labelpos:{\labelcontent}}](c) at (\nodename);
    \end{scope}
  \end{scope}
\end{tikzpicture}    
  \end{minipage}
  \begin{minipage}{0.32\linewidth}  \centering
    \colorlet{mivertexcolor}{red}
\colorlet{jivertexcolor}{black!80}
\colorlet{vertexcolor}{black!80}
\colorlet{bordercolor}{black!80}
\colorlet{linecolor}{gray}
\tikzset{vertexbase/.style={semithick, shape=circle, inner sep=2pt, outer sep=0pt, draw=bordercolor,draw opacity=0.4},%
  vertex/.style={vertexbase, fill=vertexcolor!45},%
  mivertex/.style={vertexbase, fill=mivertexcolor!45},%
  jivertex/.style={vertexbase, fill=jivertexcolor!45},%
  divertex/.style={vertexbase, top color=white, bottom color=white},%
  conn/.style={-, thick, color=linecolor}%
}
\tikzstyle{s} = [label distance=0.1cm]
\begin{tikzpicture}[scale=1.4]
  \begin{scope} 
    \begin{scope} 
      \foreach \nodename/\nodetype/\xpos/\ypos in {%
        0/mivertex/1.5/1.25,
        1/jivertex/1/2,
        2/divertex/1/0.0,
        3/divertex/0.0/1,
        4/divertex/1/1,
        5/divertex/2/1
      } \node[\nodetype] (\nodename) at (\xpos, \ypos) {};
    \end{scope}
    \begin{scope} 
      \path (0) edge[conn] (1);
      \path (3) edge[conn,draw opacity=0.4] (1);
      \path (5) edge[conn,draw opacity=0.4] (1);
      \path (2) edge[conn,draw opacity=0.4] (5);
      \path (4) edge[conn,draw opacity=0.4] (1);
      \path (2) edge[conn,draw opacity=0.4] (4);
      \path (2) edge[conn,draw opacity=0.4] (3);
    \end{scope}
    \begin{scope} 
      \foreach \nodename/\labelpos/\labelopts/\labelcontent in {%
        0/below/s/{\ \ 2, 3},
        0/above/s/{\ \ \ \ $\neg$A},
        1/below/s/{1\ \ }
      } \coordinate[label={[\labelopts]\labelpos:{\labelcontent}}](c) at (\nodename);
    \end{scope}
  \end{scope}
\end{tikzpicture}
  \end{minipage}
  \caption{Counter examples for which $\id_G$ is not a
    $(G, \{\phi_\vee\},I_{\phi_\vee})$- or
    $(G, \{\phi_\neg\},I_{\phi_\neg})$-measure of a $\context$. The
    conflicting extents are marked in red.}
  \label{fig:counter}
\end{figure}

This result raises the question for which formulas $\phi$ is $\id_{G}$
a $(G, \{\phi\},I_\phi)$-measure of $\context$. Counter examples for
which $\id_G$ is not a $(G, \{\phi_\vee\},I_{\phi_\vee})$- or $(G,
\{\phi_\neg\},I_{\phi_\neg})$-measure of a $\context$ are depicted in
\cref{fig:counter}.

\begin{corollary}[Conjunctive Logical Scale-Measures]\label{prop:clattr}
  Let $\context=(G,M,I)$ be a formal context and $\phi_\wedge\in
  \mathcal{L}(M,\{\wedge\})$, then
  $(\id_G,(G,\{\phi_\wedge\},I_{\phi_{\wedge}}))\in \Sh(\context)$.
\end{corollary}
\begin{proof}
  According to \cref{lem:deri} $(\phi_\wedge)^{I_{\phi_{\wedge}}} =
  \Var(\phi)^{I}$, hence, by \cref{prop:logiattr}
  $(\id_G,(G,\{\phi_\wedge\},I_{\phi_{\wedge}}))\in
  \Sh(\context)$.\qed
\end{proof}

\subsection{Context Apposition for Scale Construction}\label{sec:apos}
To build more complex scale-measures we employ the apposition operator
of contexts and transfer it to the realm of scale-measures. We remind
the reader that the apposition of two contexts $\context_1,\context_1$
with $G_1=G_2$ and $M_1\cap M_2 = \emptyset$ is defined as $\context_1
\app \context_2\coloneqq (G,M_1\cup M_2,I_1\cup
I_2)$. The set of extents of $\context_1 \app \context_2$ is
known to be the set of all pairwise extents of $\context_1$ and
$\context_2$. In the case of $M_{1}\cap M_{2}\neq \emptyset$ the
apposition is defined alike by coloring the attribute sets. 

\begin{definition}[Apposition of Scale-Measures]\label{def:app}
  Let $(\sigma,\Scon),(\psi,\Tcon)$ be scale-measures of
  $\context$. Then the \emph{apposition of scale-measures}
  $(\sigma,\Scon)\app(\psi,\Tcon)$ is:
  \begin{equation*}
    (\sigma,\Scon)\app(\psi,\Tcon)\coloneqq
    \begin{dcases}
      (\sigma, \Scon\,\app\Tcon)&\text{if}\ G_\Scon=G_\Tcon, \sigma =
      \psi\\
      (\sigma,\Scon)\vee(\psi,\Tcon)&\text{else}
    \end{dcases}
  \end{equation*}
\end{definition}

Note that also in the case of $G_\Scon=G_\Tcon, \sigma =\psi$ is the
scale-measure apposition is as well a join up to equivalence in the
scale-hierarchy, cf.~\cref{prop:lattice}. 

\begin{proposition}[Apposition Scale-Measure]\label{prop:app}
  Let $(\sigma,\Scon),(\psi,\Tcon)$ be two scale-measures of
  $\context$. Then $(\sigma,\Scon)\app(\psi,\Tcon)\in \Sh(\context)$.
\end{proposition}
\begin{proof}
  \begin{inparaenum}
  \item In the first case we know that set of extents $\Ext(\Scon\mid
    \Tcon)$ contains all intersections $A\cap B$ for $A\in
    \Ext(\Scon)$ and $B\in \Ext(\Tcon)$ \cite{fca-book}. Furthermore,
    we know that we can represent $\sigma^{-1}(A\cap
    B)=\sigma^{-1}(A)\cap \sigma^{-1}(B)=\sigma^{-1}(A)\cap
    \psi^{-1}(B)$. Since
    $\sigma^{-1}(\Ext(\Scon)),\psi^{-1}(\Ext(\Tcon))\subseteq
    \Ext(\context)$, we can infer that the intersection
    $\sigma^{-1}(A)\cap \psi^{-1}(B)\in\Ext(\context)$.
  \item The second case follows from~\cref{prop:lattice}.
  \end{inparaenum}\qed
\end{proof}

The apposition operator combines two scale-measures, and therefore two
views, on a data context to a new single one. We may note that the
special case of $(\sigma,\Scon)=(\id_G,\context)$ was already discussed
by Ganter and Wille~\cite{fca-book}.

\begin{proposition}\label{prop:attr}
  Let $\context = (G,M,I)$ and $\Scon =
  (G_{\Scon},M_{\Scon},I_{\Scon})$ be two formal contexts and $\sigma:
  G \to G_{\Scon}$, then TFAE:
  \begin{enumerate}[i)]
  \item $\sigma \text{ is a } \Scon\text{-measure of }\context$
  \item $\sigma \text{ is a } (G_{\Scon},\{n\},I_{\Scon}\cap
    (G_{\Scon}\times\{n\}))\text{-measure of }\context\ \text{for
      all}\ n\in M_{\Scon} $
  \end{enumerate}
\end{proposition}
\begin{proof}
  \begin{description}
  \item[$(i)\Rightarrow (ii):$] Assume $\hat n\in M_{\Scon}$ s.t.\
    $\sigma$ is not a $(G_{\Scon},\{\hat n\},
    \overbrace{I_{\Scon}\cap(G_{\Scon}\times \{\hat
      n\})}^{J})$-measure of $\context$. Then the only non-trivial
    extent $\{\hat n\}^{J}$ has a preimage $\sigma^{-1}(\{\hat
    n\}^{J})\not\in \Ext(\context)$. Since $\{\hat n\}^{J}\in
    \Ext(\Scon)$ we can conclude that $\sigma$ is not a $\Scon$-measure of
    $\context$.
  \item[$(ii)\Rightarrow (i):$] From~\cref{prop:app} follows
    $\app_{n\in M_{\Scon}}(\sigma,(G_{\Scon}
    ,\{n\},I_{\Scon}\cap(G_{\Scon}\times\{n\})))$ is again a
    scale-measure. Furthermore, by~\cref{def:app} we know that 
    $\Scon=\app_{n\in M_{\Scon}}(G_{\Scon}
    ,\{n\},I_{\Scon}\cap(G_{\Scon}\times\{n\}))$.\qed
  \end{description}
\end{proof}

\begin{corollary}[Deciding the Scale-measure Problem]
  \label{cor:decide}
  Given a formal context $(G,M,I)$ and scale-context
  $\Scon\coloneqq(G_{\Scon},M_{\Scon},I_{\Scon})$ and a map $\sigma:G\to
  G_{\Scon}$, deciding if $(\sigma,\Scon)$ is a scale-measure of
  $\context$ is in $P$. More specifically, to answer this question does
  require $O(|M_{\Scon}|\cdot|G_{\Scon}|\cdot|G|\cdot|M|)$.
\end{corollary}

We may not that this result is favorable since the naive solution
would be to compute $\Ext(\Scon)$, which is potentially exponential in
the size of $\Scon$, and checking all its elements in $\context$ for
their closure, which consumes $O(|G|\cdot|M|)$ for all
$A\in\Ext(\Scon)$. Moreover, if the formal context $\context$ is fixed
as well as $G_{\Scon}$, the computational cost for deciding the
scale-measure problem grows linearly in $|M_{\Scon}|$. Altogether,
this enables a feasible navigation in the scale-hierarchy. 

\begin{corollary}[Attribute Projection]\label{prop:projection}
  Let $\context=(G,M,I)$ be a formal context, $M_{\Scon}\subseteq M$,
  and $I_{\Scon}\coloneqq I\cap (G\times M_{\Scon})$, then
  $\sigma=\id_G$ is a $(G,M_{\Scon},I_{\Scon})$-measure of $\context$.
\end{corollary}
\begin{proof}
  The map $id_G$ is a $\context$-measure of $\context$, hence $id_G$
  is a $(G,\{n\},I\cap(G\times \{n\}))$-measure of $\context$ for
  every $n\in M$, and in particular $n\in M_{\Scon}$, by
  \cref{prop:attr}, leading to $(\id_{G},(G,M_{\Scon},I_{\Scon}))$
  being a scale-measure of $\context$, cf.~\cref{prop:app}.\qed
\end{proof}

Due to duality one may also investigate an object projection based on
the just presented attribute projection. However, an investigation of
dualities in the realm of scale-measures is deemed future work.
Combining our results on scale-measure apposition (\cref{prop:app})
with the logical attributes (\cref{prop:logiattr}) we are now tackle
the navigation problem as stated in \cref{problem:navi}.

When we look at this problem again, we find that in its generality it
does not always permit a solution. For example, consider the
well-known Boolean formal context
$\mathbb{B}_{n}\coloneqq([n],[n],\neq)$, a standard scale context,
where $[n]\coloneqq\{1,\dotsc,n\}$ and $n>2$. This context allows a
scale-measure into the standard nominal scale
$\mathbb{N}_{n}\coloneqq([n],[n],=)$, the map $\id_{[n]}$. Restricted
to any disjunctive combination of attributes, i.e.,
$M_{\Tcon}\subseteq\mathcal{L}(M,\{\vee\})$, the afore mentioned
scale-measure does not have an equivalent logical scale-measure
$(\psi,\Tcon\coloneqq([n],M_{\Tcon},I_{\Tcon}))$. This is due to the
fact that
\begin{inparaenum}
\item in nominal contexts there is for every object $g$ there is an attribute
$m$, such that ${m}'={g}$, also $|{m}'|=1$,
\item all attribute derivations in Boolean context $\mathbb{B}_{n}$
  are of cardinality $n-1$, 
\item the derivation of a disjunctive formula (over $[n]$) is the
  union of the elemental attribute derivations (\cref{lem:deri}).
\end{inparaenum} Hence, the derivation of an disjunctive formula is
at least of cardinality $n-1$ in $\Tcon$ and therefore there must not
exist an $m\in M_{\Tcon}$ such that $|\{m\}^{I_{\Tcon}}|=1$, and
therefore $\Ext(\mathbb{N})\neq \Ext(\Tcon)$.

Despite this result, we may also report positive answers for
particular instances of~\cref{problem:navi} that use conjunctive
formulas for $M_{\Tcon}$.

\begin{proposition}[Conjunctive Normalform of Scale-measures] \label{lem:appconst}
  Let $\context$ be a context, $(\sigma,\Scon)\in \Sh(\context)$. Then 
  the scale-measure  $(\psi,\Tcon)\in \Sh(\context)$ given by
  \[\psi = \id_G\quad \text{ and }\quad \Tcon = \app\limits_{A\in\sigma^{-1}(\Ext(\Scon))} (G,\{\phi =
    \wedge\ A^{I}\},I_{\phi}) \] is equivalent to $(\sigma,\Scon)$ and
  is called \emph{conjunctive normalform of} $(\sigma,\Scon)$.
  \end{proposition}
  \begin{proof}
    We know that every formal context $(G,\{\phi=\wedge
    A^{I}\},I_{\phi})$ together with $\id_{G}$ is a scale-measure
    (\cref{prop:clattr}). Moreover, every apposition of scale-measures
    (for some formal context $\context$) is again a scale-measure
    (\cref{prop:app}). Hence, the resulting $(\psi,\Tcon)$ is a
    scale-measure of $\context$.

    It remains to be shown that $\sigma^{-1}(\Ext(\Scon)) =
    \id_G(\Ext(\Tcon))$.  Scale-measure equivalence holds if
    $(\psi,\Tcon)$ reflects the same set of extents in
    $\Ext(\context)$ as $(\sigma,\Scon)$, thus if Each $(G,\{\phi =
    \wedge A^{I}\},I_{\phi})$ has the extent set $\{G,(\wedge
    A^{I})^{I_\phi}\}$. In this set we find that $(\wedge
    A^{I})^{I_\phi}=A$ by \cref{lem:deri}. Due to the apposition
    property the resulting context has the intersections of all
    subsets of $\sigma^{-1}(\Ext(\Scon))$ as extents. This set is
    closed under intersection. Therefor, $\sigma^{-1}(\Ext(\Scon)) =
    \id_G(\Ext(\Tcon))$. \qed
\end{proof}

The conjunctive normalform $(\psi,\Tcon)$ of a scale-measure
$(\sigma,\Scon)$ may constitute a more human-accessible representation
of the same scaling information. To demonstrate this in a more
practical manner we applied our method to the well-known \emph{Zoo}
data set by R.\,S.\,Forsyth, which we obtained from the UCI
repository~\cite{zoods}. For this we computed a canonical scale-measure
(\cref{lem:cssm}), for which we computed an equivalent scale-measure
(\cref{fig:zoo}) according to \cref{lem:appconst}. In the presented
example we see that the intent of animal taxons emerge naturally,
which are indicated using red colored names in~\cref{fig:zoo},
(instead of extents as used by the canonical representation).

\subsection{Order Dimension of Scale-measures}
An important property of formal contexts, and therefore of
scale-measures, is the order dimension (\cref{def:orddim}). We
already motivated their investigation with respect to our running
example, specifically the decrease of dimension
(\cref{bjicemeasure}). The substantiate formally our experimental
finding we investigate the correspondence between order dimension and
scale-hierarchies. For this we employ the Ferrers dimension of
contexts, which is equal to their order dimension~\cite[Theorem
46]{fca-book}.  A \emph{Ferrers relation} is a binary relation
$F\subseteq G\times M$ such that for $(g,m),(h,n)\in F$ it holds that
$(g,n)\not\in F \Rightarrow (h,m)\in F$. The \emph{Ferrers dimension}
of the formal context $\context$ is equal to the minimum number of
ferrers relations $F_t\subseteq G\times M, t\in T$ such that
$I=\bigcap_{t\in T} F_t$.

\begin{proposition}\label{prop:dim}
  For a context $\context$ and scale-measures
  $(\sigma,\Scon),(\psi,\Tcon)\in \SH(\context)$ with
  $(\sigma,\Scon)\leq (\psi,\Tcon)$, where $\sigma$ and $\psi$ are
  surjective, it holds that $\dim(\Scon)\leq \dim(\Tcon)$.
\end{proposition}
\begin{proof}
  We know that $(\sigma,\Scon)$ has the canonical representation
  $(\id_G,\context_{\sigma^{-1}(\Ext(\Scon))})$,
  cf.~\cref{prop:eqi-scale}, and the same is true for $(\psi,\Tcon)$.
  Since $(\sigma,\Scon)\leq (\psi,\Tcon)$ it holds that
  $\sigma^{-1}(\Ext(\Scon))\subseteq \psi^{-1}(\Ext(\Tcon))$ and the
  scale $\context_{\psi^{-1}(\Ext(\Tcon))}$ restricted to the set
  $\sigma^{-1}(\Ext(\Scon))$ as attributes is equal to
  $\context_{\sigma^{-1}(\Ext(\Scon))}$. Hence, a Ferrers set $F_{T}$
  such that $\bigcap_{t\in T}F_{t}$ is equal to the incidence of
  $\context_{\psi^{-1}(\Ext(\Tcon))}$, can be restricted to the
  attribute set $\sigma^{-1}(\Ext(\Scon))$ and is then equal to the
  incidence of $\context_{\sigma^{-1}(\Ext(\Scon))}$. Thus, as
  required, $\dim(\context_{\sigma^{-1}(\Ext(\Scon))}) \leq
  \dim(\context_{\psi^{-1}(\Ext(\Tcon))})$.\qed
\end{proof}

Building up on this result we can provide an upper bound for the
dimension of apposition of scale-measures for some formal context
$\context$.

\begin{proposition}
  For a context $\context$ and scale-measures
  $(\sigma,\Scon),(\psi,\Tcon)\in \SH(\context)$ with
  $(\sigma,\Scon)\app (\psi, \Tcon){=}(\delta, \Ocon)$. Then order dim.\
  of $\Ocon$ is bound by $dim(\Ocon){\leq}\dim(\Scon){+}\dim(\Tcon)$.
\end{proposition}
\begin{proof}
  Without loss of generality we consider for all scale-measures their
  canonical representation, only. Let $F_T$ be a Ferrers set of the
  formal context $\Tcon$ such that $\bigcap_{t\in T}F_{t}=I_{\Tcon}$
  and similarly $\bigcap_{s\in S} F_s=I_{\Scon}$.  For any Ferrers
  relation $F$ of $\Scon$ it follows that $F\cup (G\times M_{\Tcon})$
  is a Ferrers relation of $\Scon\,\app\Tcon$. Hence, the intersection
  of $\bigcap_{s\in S} F_s\cup (G\times M_\Tcon)$ and $\bigcap_{t\in
    T} F_t\cup (G\times M_\Scon)$ is a Ferrers set and is equal to
  $I_{\, \Scon\,\app\Tcon}$. Since this construction does neither change the
  cardinality of index set $T$ nor the index set $S$, the required
  inequality follows. \qed
\end{proof}

\section{Implications for Data Set Scaling}
We revisit the running example $\context_{\text{BJ}}$ (\cref{bjice}) and want
to outline a semi-automatically procedure to obtain a human-meaningful
scale-measure from it, as depicted in~\cref{bjicemeasure}, based on
the insights from~\cref{sec:methods}. In this example, we derive new
attributes $M_{\Tcon}\subseteq \mathcal{L}(M,\{\wedge,\vee,\neg\})$ from
the original attribute set $M$ of $\context_{\text{BJ}}$ using background
knowledge. This process results in
\begin{align*}
M_{\Tcon}=\{&\textbf{Choco}=\texttt{Choco Ice}\vee\texttt{Choco
              Pieces},\\
  &\textbf{Caramel}=\texttt{Caramel
    Ice}\vee\texttt{Caramel},\\
  &\textbf{Peanut}=\texttt{Peanut
    Ice}\vee\texttt{Peanut Butter},\\
  &\textbf{Brownie, Dough, Vanilla}\}
\end{align*}

Such propositional features can be bear various meanings, in our
example we interpret $M_{\Tcon}$ as taste attributes (as opposed to
ingredients). Another possible set $M_{\Tcon}$ could represent
ingredient mixtures ($\wedge$) to generate a recipe view on the
presented ice creams. From $M_{\Tcon}$ we can now derive
semi-automatically a scale-measure (\cref{prop:app,prop:attr}) if it
exists (\cref{cor:decide}).

\subsection{Scaling of Larger Data Set}
To demonstrate the benefits of the scale-measure navigation on a
larger data set, we evaluate our method on a data set that related
spices to dishes~\cite{pqcores,herbs}. We decided for another food
related data set, since we assume that this knowledge domain is easily
to grasp. Specifically, the data set is comprised of 56 dishes as
objects and 37 spices as their attributes, and the resulting context
is in the following denoted by $\context_{\text{Spices}}$. The dishes
in the data set are picked from multiple categories, such as
vegetables, meat, or fish dishes. The incidence
$I_{\context_{\text{Spices}}}$ indicates that a spice $m$ is necessary
to cook a dish $g$. The concept lattice of $\context_{\text{Spices}}$
has 421 concepts and is therefore too large for a meaningful human
comprehension. Thus, using scale-measures through our methods, we are
able to generate two small-scaled views of readable size. Both scales,
as depicted in~\cref{fig:gewscale}, measure the dishes in terms of
spice mixtures $M_{\Tcon}\subseteq\mathcal{L}(M,\{\wedge\})$. For the
conjunction of spices we transformed intent sets
$B\in\Int(\context_{\text{Spices}})$ to propositional formulas
$\bigwedge_{m\in B} m$. However, in order to retrieve a small scale
context we decided for using intents with high support, only, i.e.,
$B'/G$ is high with respect to some selection criterion. We employed
two different selection criteria: A) high support in all dishes; B)
high support in meat dishes. Afterwards we derive semi-automatically
two scale-measures (\cref{prop:app,prop:attr}).  Both scale-measures
include five spices mixtures. The concept lattice for the scale
context of A) is depicted in~\cref{fig:gewscale} (bottom), and for B)
in~\cref{fig:gewscale} (top). We named all selected intent sets to
make them more easily addressable. Both scales can be used to identify
similar flavored dishes, e.g., a menu such as \emph{deer} in
combination with \emph{red cabbage}, which share the \emph{bay leaf
  mix}. Based on the scale-measures one might be interested to further
navigate in the scale-hierarchy by adding additional spice mixtures
(\cref{prop:app}), or employing other selection criterion, which
result in different views on the data set $\context_{\text{Spices}}$,
e.g., vegetarian. 

Finally, we may point out that in contrast to feature compression
techniques, such as LSA (which use linear combinations of attributes),
the scale-measure attributes are directly interpretable by the
semantics of propositional logics on the  original data set attributes.

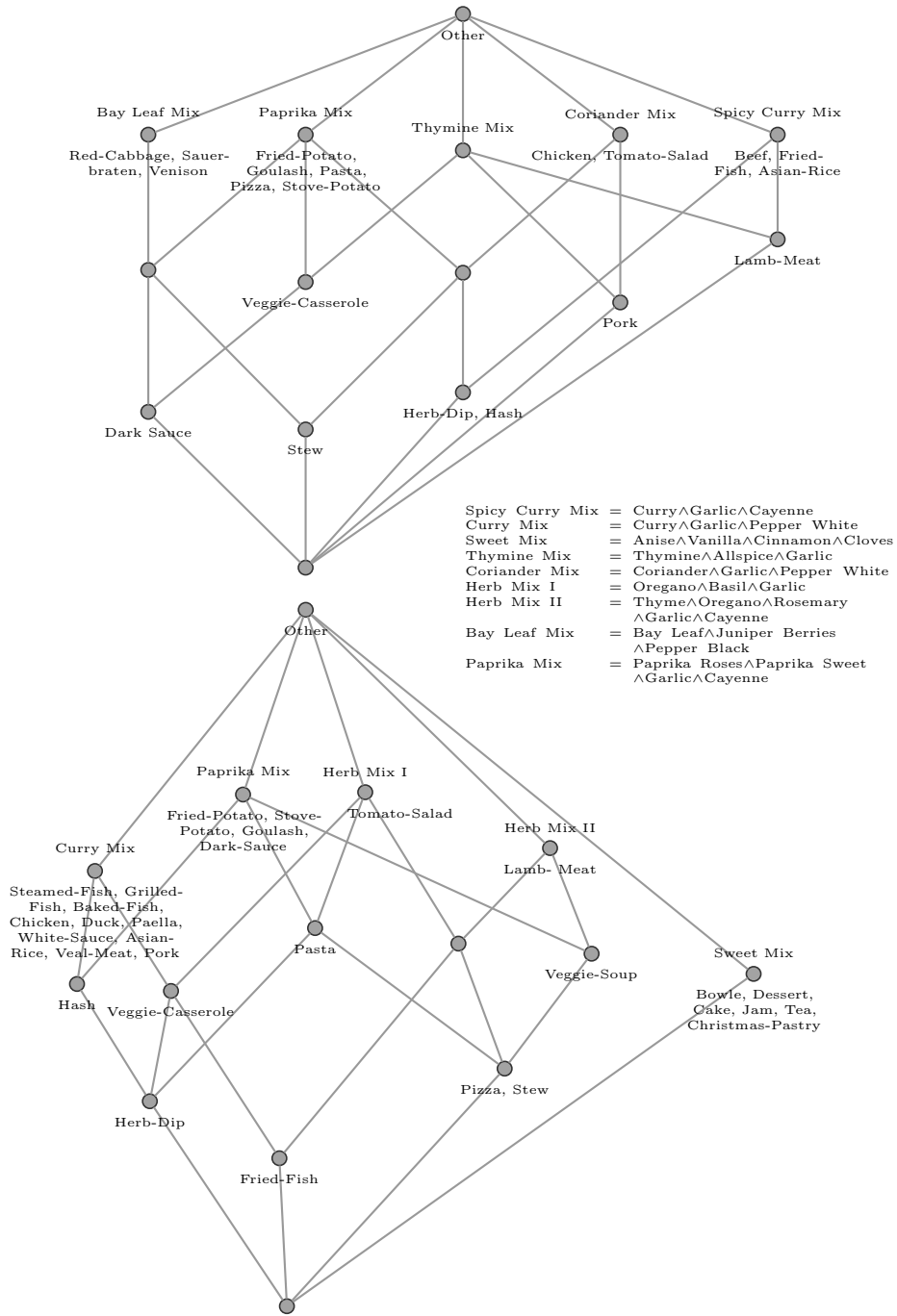
\begin{figure}
\hspace{-1cm}\begin{tikzpicture}[]
  \node at (0,2.2) {\scalebox{1}{\colorlet{mivertexcolor}{black!80}
\colorlet{jivertexcolor}{black!80}
\colorlet{vertexcolor}{black!80}
\colorlet{bordercolor}{black!80}
\colorlet{linecolor}{gray!80}
\tikzset{vertexbase/.style={semithick, shape=circle, inner sep=2pt, outer sep=0pt, draw=bordercolor},%
  vertex/.style={vertexbase, fill=vertexcolor!45},%
  mivertex/.style={vertexbase, fill=mivertexcolor!45},%
  jivertex/.style={vertexbase, fill=jivertexcolor!45},%
  divertex/.style={vertexbase, top color=mivertexcolor!45, bottom color=jivertexcolor!45},%
  conn/.style={-, thick, color=linecolor}%
}
\tikzstyle{s} = [text width=2.5cm,align=center,label distance=1mm]
\begin{tikzpicture}[scale=0.43,font=\tiny]
  \begin{scope} 
    \begin{scope} 
      \foreach \nodename/\nodetype/\xpos/\ypos in {%
        0/vertex/-4/3.1189189189189186,
        1/jivertex/-4/7.548648648648651,
        2/jivertex/-9/8.113513513513514,
        3/jivertex/1/8.737837837837837,
        4/jivertex/6/11.621621621621621,
        5/jivertex/-4/12.275675675675675,
        6/vertex/1/12.572972972972973,
        7/vertex/-9/12.662162162162165,
        8/jivertex/11/13.643243243243244,
        9/mivertex/1/16.4972972972973,
        10/divertex/-9/17,
        11/mivertex/11/17,
        12/mivertex/6/17,
        13/mivertex/-4/17,
        14/vertex/1/20.86756756756757
      } \node[\nodetype] (\nodename) at (\xpos, \ypos) {};
    \end{scope}
    \begin{scope} 
      \path (1) edge[conn] (7);
      \path (2) edge[conn] (7);
      \path (13) edge[conn] (14);
      \path (11) edge[conn] (14);
      \path (5) edge[conn] (9);
      \path (3) edge[conn] (6);
      \path (10) edge[conn] (14);
      \path (5) edge[conn] (13);
      \path (6) edge[conn] (12);
      \path (0) edge[conn] (8);
      \path (7) edge[conn] (13);
      \path (0) edge[conn] (4);
      \path (0) edge[conn] (3);
      \path (8) edge[conn] (11);
      \path (1) edge[conn] (6);
      \path (8) edge[conn] (9);
      \path (4) edge[conn] (12);
      \path (6) edge[conn] (13);
      \path (0) edge[conn] (1);
      \path (7) edge[conn] (10);
      \path (9) edge[conn] (14);
      \path (0) edge[conn] (2);
      \path (4) edge[conn] (9);
      \path (3) edge[conn] (11);
      \path (2) edge[conn] (5);
      \path (12) edge[conn] (14);
    \end{scope}
    \begin{scope} 
      \foreach \nodename/\labelpos/\labelopts/\labelcontent in {%
        1/below/s/{Stew},
        2/below/s/{Dark Sauce},
        3/below/s/{Herb-Dip, Hash},
        4/below/s/{Pork},
        5/below/s/{Veggie-Casserole},
        8/below/s/{Lamb-Meat},
        9/above/s/{Thymine Mix},
        10/below/s/{Red-Cabbage, Sauerbraten, Venison},
        10/above/s/{Bay Leaf Mix},
        11/below/s/{Beef, Fried-Fish, Asian-Rice},
        11/above/s/{Spicy Curry Mix},
        12/below/s/{Chicken, Tomato-Salad},
        12/above/s/{Coriander Mix},
        13/below/s/{Fried-Potato, Goulash, Pasta, Pizza, Stove-Potato},
        13/above/s/{Paprika Mix},
        14/below/s/{Other}
      } \coordinate[label={[\labelopts]\labelpos:{\labelcontent}}](c) at (\nodename);
    \end{scope}
  \end{scope}
\end{tikzpicture}

}};
  \node[text width = 10cm] at (5.2,-2) {\tiny
\setlength{\tabcolsep}{2pt}
    \begin{tabular}{lcl}
      Spicy Curry Mix&  =& Curry$\wedge$Garlic$\wedge$Cayenne \\ 
      Curry Mix& =& Curry$\wedge$Garlic$\wedge$Pepper White\\  
      Sweet Mix& = & Anise$\wedge$Vanilla$\wedge$Cinnamon$\wedge$Cloves\\
      Thymine Mix &= &Thymine$\wedge$Allspice$\wedge$Garlic\\
      Coriander Mix& =& Coriander$\wedge$Garlic$\wedge$Pepper White\\
      Herb Mix I &= &Oregano$\wedge$Basil$\wedge$Garlic\\
      Herb Mix II&=&Thyme$\wedge$Oregano$\wedge$Rosemary\\&&$\wedge$Garlic$\wedge$Cayenne\\
Bay Leaf Mix& = &Bay Leaf$\wedge$Juniper Berries\\&&$\wedge$Pepper
                   Black\\
Paprika Mix& = & 
       Paprika Roses$\wedge$Paprika Sweet\\&&$\wedge$Garlic$\wedge$Cayenne
    \end{tabular}
  };
  \node at (-0.65,-7) {\scalebox{1}{\colorlet{mivertexcolor}{black!80}
\colorlet{jivertexcolor}{black!80}
\colorlet{vertexcolor}{black!80}
\colorlet{bordercolor}{black!80}
\colorlet{linecolor}{gray!80}
\tikzset{vertexbase/.style={semithick, shape=circle, inner sep=2pt, outer sep=0pt, draw=bordercolor},%
  vertex/.style={vertexbase, fill=vertexcolor!45},%
  mivertex/.style={vertexbase, fill=mivertexcolor!45},%
  jivertex/.style={vertexbase, fill=jivertexcolor!45},%
  divertex/.style={vertexbase, top color=mivertexcolor!45, bottom color=jivertexcolor!45},%
  conn/.style={-, thick, color=linecolor}%
}
\tikzstyle{s} = [text width=2.5cm,align=center,label distance=1mm]
\begin{tikzpicture}[scale=0.36,font=\tiny]
  \begin{scope} 
    \begin{scope} 
      \foreach \nodename/\nodetype/\xpos/\ypos in {%
        0/vertex/-0.7148648648648681/1.3351351351351362,
        1/jivertex/-1/7,
        2/jivertex/-5.91756756756757/9.183783783783785,
        3/jivertex/7.549999999999997/10.432432432432435,
        4/vertex/-5.11486486486487/13.405405405405407,
        5/jivertex/-8.682432432432435/13.672972972972975,
        6/divertex/17/14.059459459459461,
        7/jivertex/10.849999999999994/14.832432432432434,
        8/vertex/5.795945945945942/15.21891891891892,
        9/vertex/0.3554054054054028/15.813513513513515,
        10/mivertex/-8/18,
        11/mivertex/9.274324324324319/18.875675675675677,
        12/mivertex/-2.379729729729732/20.92702702702703,
        13/mivertex/2.258108108108104/21.01621621621622,
        14/vertex/0/28
      } \node[\nodetype] (\nodename) at (\xpos, \ypos) {};
    \end{scope}
    \begin{scope} 
      \path (2) edge[conn] (9);
      \path (12) edge[conn] (14);
      \path (3) edge[conn] (8);
      \path (0) edge[conn] (3);
      \path (10) edge[conn] (14);
      \path (2) edge[conn] (5);
      \path (4) edge[conn] (13);
      \path (9) edge[conn] (13);
      \path (11) edge[conn] (14);
      \path (9) edge[conn] (12);
      \path (1) edge[conn] (4);
      \path (7) edge[conn] (11);
      \path (0) edge[conn] (6);
      \path (5) edge[conn] (12);
      \path (0) edge[conn] (2);
      \path (2) edge[conn] (4);
      \path (8) edge[conn] (13);
      \path (1) edge[conn] (8);
      \path (6) edge[conn] (14);
      \path (0) edge[conn] (1);
      \path (4) edge[conn] (10);
      \path (5) edge[conn] (10);
      \path (3) edge[conn] (7);
      \path (3) edge[conn] (9);
      \path (13) edge[conn] (14);
      \path (8) edge[conn] (11);
      \path (7) edge[conn] (12);
    \end{scope}
    \begin{scope} 
      \foreach \nodename/\labelpos/\labelopts/\labelcontent in {%
        1/below/s/{Fried-Fish},
        2/below/s/{Herb-Dip},
        3/below/s/{Pizza, Stew},
        4/below/s/{Veggie-Casserole},
        5/below/s/{Hash},
        6/below/s/{Bowle, Dessert, Cake, Jam, Tea, Christmas-Pastry},
        6/above/s/{Sweet Mix},
        7/below/s/{Veggie-Soup},
        9/below/s/{Pasta},
        10/below/s/{Steamed-Fish, Grilled-Fish, Baked-Fish, Chicken, Duck, Paella, White-Sauce, Asian-Rice, Veal-Meat, Pork},
        10/above/s/{Curry Mix},
        11/below/s/{Lamb- Meat},
        11/above/s/{Herb Mix II},
        12/below/s/{Fried-Potato, Stove-Potato, Goulash, Dark-Sauce},
        12/above/s/{Paprika Mix},
        13/below/s/{\phantom{aaaaaaaa}Tomato-Salad},
        13/above/s/{Herb Mix I},
        14/below/s/{Other}
      } \coordinate[label={[\labelopts]\labelpos:{\labelcontent}}](c) at (\nodename);
    \end{scope}
  \end{scope}
\end{tikzpicture}

}};
\end{tikzpicture}
\caption{In this figure, we display the concept lattices of two scale
  contexts for which the identity map is a scale-measures of the
  spices context. The attributes of the scales are spice mixtures
  generated by propositional logic. By \emph{Other} we identify all
  objects in the top concept for better readability. \\}
  \label{fig:gewscale}
\end{figure}

\section{Related Work}
Measurement is an important field of study in many (scientific)
disciplines that involve the collection and analysis of
data. According to \citeauthor{stevens1946theory}
\cite{stevens1946theory} there are four feature categories that can be
measured, i.e. \emph{nominal}, \emph{ordinal}, \emph{interval} and
\emph{ratio} features. Although there are multiple extensions and
re-categorizations of the original four categories, e.g., most
recently~\citeauthor{Chrisman} introduced ten \cite{Chrisman}, for the
purpose of our work the original four suffice. Each of these
categories describe which operations are supported per feature
category. In the realm of formal concept analysis we work often with
\emph{nominal} and \emph{ordinal} features, supporting value
comparisons by $=$ and $<,>$. Hence grades of detail/membership
cannot be expressed. A framework to describe and analyze the
measurement for Boolean data sets has been introduced in
\cite{cmeasure} and \cite{scaling}, called \emph{scale-measures}. It
characterizes the measurement based on object clusters that are formed
according to common feature (attribute) value combinations. An
accompanied notion of dependency has been studied \cite{manydep},
which led to attribute selection based measurements of boolean
data. The formalism includes a notion of consistency enabling the
determination of different views and abstractions, called
\emph{scales}, to the data set. This approach is comparable to
\emph{OLAP}~\cite{olap} for databases, but on a conceptual
level. Similar to the feature dependency study is an approach for
selecting relevant attributes in contexts based on a mix of lattice
structural features and entropy
maximization~\cite{DBLP:conf/iccs/HanikaKS19}. All discussed
abstractions reduce the complexity of the data, making it easier to
understand by humans.

Despite the in this work demonstrated expressiveness of the
scale-measure framework, it is so far insufficiently studied in the
literature.  In particular algorithmical and practical calculation
approaches are missing. Comparable and popular machine learning
approaches, such as feature compressed techniques, e.g., \emph{Latent
  Semantic Analysis} \cite{lsa,lsaapp}, have the disadvantage that the
newly compressed features are not interpretable by means of the
original data and are not guaranteed to be consistent with said
original data. The methods presented in this paper do not have these
disadvantages, as they are based on meaningful and interpretable
features with respect to the original features using propositional
expressions. In particular preserving consistency, as we did, is not a
given, which was explicitly investigated in the realm scaling
many-valued formal contexts~\cite{logiscale} and implicitly studied for generalized
attributes~\cite{leonardOps}.

Earlier approaches to use scale contexts for complexity reduction in
data used constructs such as $(G_N\subseteq \mathcal{P}(N),N,\ni)$ for
a formal context $\context=(G,M,I)$ with $N\subseteq M$ and the
restriction that at least all intents of $\context$ restricted to $N$
are also intent in the scale~\cite{stumme99hierarchies}. Hence, the
size of the scale context concept lattice depends directly on the size
of the concept lattice of $\context$. This is particularly infeasible
if the number of intents is exponential, leading to incomprehensible
scale lattices. This is in contrast to the notion of scale-measures,
which cover at most the extents of the original context, and can
thereby display selected and interesting object dependencies of
scalable size.

\section{Conclusion}
Our work has broadened the understanding of the data scaling process
and has paved the way for the development of novel scaling algorithms,
in particular for Boolean data, which we summarize under the term
\emph{Exploring Conceptual Measurements}. We build our framework on
the notion of scale-measures, which themselves are interpretations of
formal contexts. By studying and extending the theory on
scale-measures, we found that the set of all possible measurements for
a formal context is lattice ordered, up to equivalence. Thus, this set
is navigable using the lattice's meet and join operations.
Furthermore, we found that the problem of deciding whether for a given
formal context $\context$ and a tuple $(\sigma,\Scon)$ the latter represents a
scale-measure for the former is PTIME with respect to the respective object
and attribute set sizes. All this and the following is based on 
our main result that for a given formal context $\context$ the set of all
scale-measures and the set of all sub-closure systems of
$\BV(\context)$ are isomorphic. 

To ensure our goal for human comprehensible scaling we derived a
propositional logic scaling of formal contexts by transferring and
extending results from conceptual scaling~\cite{logiscale}. With this
approach, we are able to introduce new features that lead to
interpretable scale features in terms of a logical formula and with
respect to the original data set attributes. Moreover, these features
are suitable to create any possible scale measurement of the
data. Finally, we found that the order dimension decreases
monotonously when scale-measures are coarsened, hinting the principal
improved readability of scale-measures in contrast to the original
data set.  We have substantiated our theoretical results with three
exemplary data analyses. In particular we demonstrated that employing
propositional logic on the attribute set enables us to express and
apply meaningful scale features, which improved the human readability
in a natural manner. All methods used throughout this work are
published with the open source software
\texttt{conexp-clj}\cite{conexp}, a research tool for Formal Concept
Analysis.

We identified three different research directions for future work,
which together may lead to an efficient and comprehensible data
scaling framework. First of all, the development of meaningful
criteria for ranking or valuing scale-measures is necessary. Although
our results enable an efficient navigation in the lattice of
scale-measures, it cannot provide a promising direction, except from
decreasing the order dimension. Secondly, efficient algorithms for
computing an initial, well ranked/rated scale-measure and the
subsequent navigation are required. Even though we showed a bound for
the computational run time complexity, we assume that this can still
be improved. Thirdly, a natural approach for decreasing the
computational cost of navigating conceptual measurements would be to
employ a set of minimal closure generators instead of the closure
system. We speculate that our results hold in this case. Yet, it is an
open questions if procedures, such as TITANIC~\cite{titanic}, can be
adapted to efficiently navigate the scale-hierarchy of a formal context.

\bibliography{paper}

\newpage
\appendix
\section{Example}
\begin{figure}[h!]
  \centering
  \colorlet{mivertexcolor}{black!80}
\colorlet{jivertexcolor}{black!80}
\colorlet{vertexcolor}{black!80}
\colorlet{bordercolor}{black!80}
\colorlet{linecolor}{gray!80}
\colorlet{highlightline}{gray!90!black}
\tikzset{vertexbase/.style={semithick, shape=circle, inner sep=2pt, outer sep=0pt, draw=bordercolor},%
  vertex/.style={vertexbase, fill=vertexcolor!45},%
  mivertex/.style={vertexbase, fill=mivertexcolor!45},%
  jivertex/.style={vertexbase, fill=jivertexcolor!45},%
  divertex/.style={vertexbase, top color=mivertexcolor!45, bottom color=jivertexcolor!45},%
  conn/.style={-, thick, color=linecolor}%
}
\tikzstyle{s} = [text width=2.7cm,align=center,label distance=0.5mm,rotate=90]
\begin{tikzpicture}[scale=0.27,font=\scriptsize,rotate=90]
  \begin{scope} 
    \begin{scope} 
      \foreach \nodename/\nodetype/\xpos/\ypos in {%
        0/vertex/-14.10696752808886/2.063112932752604,
        1/divertex/-15.18638579520463/9.915624062500974,
        2/jivertex/-16.53128035943854/15.081425960495068,
        3/jivertex/-21.43857319782248/15.515947789631511,
        4/jivertex/-27.02528242957675/17.254035106177284,
        5/jivertex/-0.2711526641756876/17.564407841274743,
        6/jivertex/-12.686062068074094/18.867973328684084,
        7/vertex/-17.900324017711426/20.171538816093413,
        8/jivertex/13.199024039054073/20.543986098210368,
        9/jivertex/-5.361265519774037/23.08904252600954,
        10/jivertex/-23.300809608407228/23.15111707302903,
        11/jivertex/-18.769367675984313/23.29941526110699,
        12/mivertex/-0.20907811715620284/23.771862543223953,
        13/vertex/8.853805747689641/24.640906201496843,
        14/vertex/-14.237925743561398/26.352769876984144,
        15/divertex/22.882653374094843/27.789528116705337,
        16/vertex/3.1429474218963662/29.79309360411468,
        17/mivertex/-5.175041878715554/29.855168151134162,
        18/vertex/-15.231118495873268/30.848360903446036,
        19/jivertex/9.2262530298066/31.59325546767994,
        20/vertex/18.102913253593947/31.903628202777405,
        21/mivertex/-26.28038786534285/33.20719369018673,
        22/mivertex/-21.68687138590044/34.75905736567404,
        23/mivertex/3.7636928920912993/36.99374105837575,
        24/mivertex/11.150563987410841/37.117890152414745,
        25/mivertex/18.351211441671914/37.28262887512204,
        26/vertex/11.416042158274398/44.042686700714795
      } \node[\nodetype] (\nodename) at (\xpos, \ypos) {};
    \end{scope}
    \begin{scope} 
      \path (24) edge[conn,draw=highlightline] (26);
      \path (25) edge[conn,draw=highlightline] (26);
      \path (18) edge[conn] (22);
      \path (6) edge[conn] (12);
      \path (8) edge[conn] (15);
      \path (4) edge[conn] (21);
      \path (9) edge[conn,draw=highlightline] (17);
      \path (12) edge[conn,draw=highlightline] (23);
      \path (15) edge[conn] (20);
      \path (5) edge[conn,draw=highlightline] (13);
      \path (0) edge[conn] (8);
      \path (1) edge[conn] (2);
      \path (23) edge[conn,draw=highlightline] (26);
      \path (10) edge[conn] (21);
      \path (3) edge[conn] (6);
      \path (11) edge[conn] (22);
      \path (6) edge[conn] (14);
      \path (14) edge[conn] (17);
      \path (14) edge[conn] (18);
      \path (5) edge[conn,draw=highlightline] (12);
      \path (16) edge[conn,draw=highlightline] (23);
      \path (0) edge[conn] (1);
      \path (4) edge[conn] (7);
      \path (19) edge[conn,draw=highlightline] (25);
      \path (20) edge[conn,draw=highlightline] (25);
      \path (8) edge[conn] (13);
      \path (7) edge[conn] (20);
      \path (0) edge[conn] (5);
      \path (3) edge[conn] (10);
      \path (18) edge[conn] (23);
      \path (13) edge[conn,draw=highlightline] (19);
      \path (19) edge[conn,draw=highlightline] (23);
      \path (11) edge[conn] (25);
      \path (7) edge[conn] (11);
      \path (2) edge[conn] (7);
      \path (0) edge[conn] (4);
      \path (16) edge[conn,draw=highlightline] (24);
      \path (17) edge[conn,draw=highlightline] (23);
      \path (13) edge[conn,draw=highlightline] (20);
      \path (10) edge[conn] (14);
      \path (0) edge[conn] (3);
      \path (22) edge[conn] (26);
      \path (2) edge[conn] (13);
      \path (0) edge[conn] (9);
      \path (9) edge[conn,draw=highlightline] (16);
      \path (13) edge[conn,draw=highlightline] (16);
      \path (20) edge[conn,draw=highlightline] (24);
      \path (21) edge[conn] (22);
      \path (2) edge[conn] (18);
    \end{scope}
    \begin{scope} 
      \foreach \nodename/\labelpos/\labelopts/\labelcontent in {%
        1/below/s/{frog, toad, newt},
        1/above/s/{{\color{red!80!black}Amphibian}},
        2/below/s/{slowworm, tuatara, pitviper},
        3/below/s/{sealion},
        4/below/s/{OtherFishes},
        4/above/s/{{\color{red!80!black} Fish}},
        5/below/s/{OtherBirds},
        5/above/s/{{\color{red!80!black}Bird}},
        6/below/s/{wallaby, squirrel, {fruitbat}, gorilla, vampire},
        8/below/s/{wasp, gnat, housefly, moth, termite, honeybee, flea, ladybird},
        8/above/s/{{\color{red!80!black}Insect}},
        9/below/s/{platypus},
        10/below/s/{porpoise, seal, dolphin},
        11/below/s/{seasnake},
        12/above/s/{two-legs},
        13/below/s/{slug, tortoise, worm},
        14/below/s/{OtherMammals},
        15/below/s/{lobster, crayfish},
        15/above/s/{six-legs},
        17/above/s/{{\color{red!80!black} Mammal}},
        19/below/s/{scorpion},
        20/below/s/{crab, clam, octopus, seawasp, starfish},
        21/above/s/{fins},
        22/above/s/{predator},
        23/above/s/{backbone},
        24/above/s/{feathers},
        25/above/s/{eggs}
      } \coordinate[label={[\labelopts]\labelpos:{\labelcontent}}](c) at (\nodename);
    \end{scope}
  \end{scope}
  \node[rotate=90] at (13,6) {\setlength{\tabcolsep}{2pt}
    \begin{tabular}{lcl}
    {\color{red!80!black}Bird}&=& eggs $\wedge$ two-legs $\wedge$ feathers \\
    {\color{red!80!black}Fish}&=& eggs $\wedge$ fins \\
    {\color{red!80!black}Insect} &=& airborne $\wedge$ six-legs\\
    {\color{red!80!black}Mammal}&=&backbone $\wedge$ not-eggs\\
    {\color{red!80!black}Amphibian}&=& not-fins $\wedge$ predator $\wedge$ feathers $\wedge$ airborne
    \end{tabular}};
\end{tikzpicture}



  \caption{Concept lattice of a scale-measure of the zoo data set with
    twenty-seven of the original 4579 concepts. Contained objects are
    animals and attributes are characteristics. Newly introduced
    logical attributes are a characterization of animal taxons. The
    objects \emph{girl,frogB} were omitted. We grouped
    {\textbf{OtherFishes}=\{seahorse, sole, herring, piranha, pike,
      chub, haddock, stingray, carp, bass, dogfish, catfish,
      tuna\},\textbf{OtherMammals}=\{reindeer, aardvark, polecat,
      wolf, mole, vole, hare, boar, cavy, antelope, goat, puma,
      mongoose, pony, bear, pussycat, lynx, elephant, calf, mink,
      opossum, leopard, buffalo, lion, giraffe, cheetah, oryx, deer,
      hamster, raccoon\},\textbf{OtherBirds}=\{gull, parakeet, crow,
      skua, swan, hawk, sparrow, lark, wren, dove, vulture, penguin,
      duck, flamingo, pheasant, rhea, ostrich, skimmer, chicken,
      kiwi\} }}
  \label{fig:zoo}
\end{figure}

\end{document}